%% file: Arxiv_V2_Submission.tex
\documentclass[twoside]{article}
\usepackage{microtype}
\usepackage{graphicx}
\usepackage{booktabs} % for professional tables

\usepackage{hyperref}

\usepackage{algorithm}
\usepackage{algorithmic}

\usepackage[round]{natbib}

\usepackage{amsfonts}
\usepackage{amsmath}
\usepackage{enumitem}
\usepackage{subcaption}

\usepackage{aistats2020}

\usepackage{amsfonts}
\usepackage{amsmath}
\usepackage{enumitem}
\usepackage{subcaption}
\usepackage{amsthm}

\usepackage[table]{xcolor} 
\newcommand{\argmin}{\operatornamewithlimits{arg\ min}}
\newcommand{\argmax}{\operatornamewithlimits{arg\ max}}

\newcommand{\beq} {\begin{equation}}
\newcommand{\eeq} {\end{equation}}
\newcommand{\beqs} {\begin{equation*}}
\newcommand{\eeqs} {\end{equation*}}

\newcommand{\EE}{\mathbb{E}}
\newcommand{\PP}{\mathbb{P}}
\newcommand{\RR}{\mathbb{R}}

\newcommand{\bS}{\boldsymbol{S}}

\newcommand{\sE}{\ensuremath{\mathcal{E}}}

\newcommand{\sH}{\ensuremath{\mathcal{H}}}

\newcommand{\sX}{\ensuremath{\mathcal{X}}}

\newcommand{\Rbb}{\mathbb{R}}

\newcommand{\n}[1]{\|#1\|}

\newcommand{\Tlambda}{\ensuremath{\tilde{\lambda}}}

\newtheorem{theorem}{Theorem}[section]
\newtheorem{lemma}[theorem]{Lemma}

\newtheorem{corollary}[theorem]{Corollary}
\newtheorem{definition}[theorem]{Definition}

\begin{document}

\twocolumn[

\aistatstitle{Simple Regret Minimization for Contextual Bandits}
\aistatsauthor{Aniket Anand Deshmukh* \footnotemark[1], Srinagesh Sharma* \footnotemark[1] \And James W. Cutler\footnotemark[2] \AND  Mark Moldwin\footnotemark[3] \And Clayton Scott\footnotemark[1]  }

\aistatsaddress{ \footnotemark[1]Department of EECS, University of Michigan, Ann Arbor, MI, USA \\ \footnotemark[2]Department of Aerospace Engineering, University of Michigan, Ann Arbor, MI, USA \\ \footnotemark[3]Climate and Space Engineering, University of Michigan, Ann Arbor, MI, USA }
]

\begin{abstract}
  There are two variants of the classical multi-armed bandit (MAB) problem that have received considerable attention from machine learning researchers in recent years: contextual bandits and simple regret minimization. Contextual bandits are a sub-class of MABs where, at every time step, the learner has access to side information that is predictive of the best arm. Simple regret minimization assumes that the learner only incurs regret after a pure exploration phase. In this work, we study simple regret minimization for contextual bandits.  Motivated by applications where the learner has separate training and autonomous modes, we assume that the learner experiences a {\em pure exploration} phase, where feedback is received after every action but no regret is incurred, followed by a {\em pure exploitation} phase in which regret is incurred but there is no feedback. We present the Contextual-Gap algorithm and establish performance guarantees on the simple regret, i.e., the regret during the pure exploitation phase. Our experiments examine a novel application to adaptive sensor selection for magnetic field estimation in interplanetary spacecraft, and demonstrate considerable improvement over algorithms designed to minimize the cumulative regret.
\end{abstract}

\input{1_introduction}

%-------------------------------------------------------------------------------------%
%\section{Motivation}
%\label{sec:motivation}
%-------------------------------------------------------------------------------------%
\input{2_motivation}

%-------------------------------------------------------------------------------------%
%\section{Formal Setting}
%\label{sec:formal_setting}
%-------------------------------------------------------------------------------------%
\input{3_formal_setting}

%\blue{We need to change reward to R or whatever you choose everywhere including supplementary material.}
%-------------------------------------------------------------------------------------%
%\section{Algorithm}
%\label{sec:algorithm}
%-------------------------------------------------------------------------------------%
\input{4_algorithm}

%-------------------------------------------------------------------------------------%
%\section{Learning Theoretic Analysis}
%\label{sec:learning_theoretic_analysis}
%-------------------------------------------------------------------------------------%
\input{5_learning_theoretic_analysis}

%-------------------------------------------------------------------------------------%
%\section{Results and Discussion} 
%\label{sec:results_and_discussion}
%-------------------------------------------------------------------------------------%
\input{6_results_discussion}

%-------------------------------------------------------------------------------------%
%\section{Conclusion and Future Work}
%\label{sec:conclusion_future_work}
%-------------------------------------------------------------------------------------%
\input{7_conclusion}

\bibliographystyle{apalike}
\bibliography{references}

\input{8_appendix}

\end{document}

%% file: 1_introduction.tex
%-------------------------------------------------------------------------------------%
\section{Introduction}
\label{sec:intro}
%-------------------------------------------------------------------------------------%
%\blue{Change all \( \delta_2 \) to \( \delta\) in sup material. }

The multi-armed bandit (MAB) is a framework for sequential decision making where, at every time step, the learner selects (or ``pulls'') one of several possible actions (or ``arms''), and receives a reward based on the selected action. The regret of the learner is the difference between the maximum possible reward and the reward resulting from the chosen action. In the classical MAB setting, the goal is to minimize the sum of all regrets, or {\em cumulative regret}, which naturally leads to an exploration/exploitation trade-off problem \citep{auer2002finite}. If the learner explores too little, it may never find an optimal arm which will increase its cumulative regret. If the learner explores too much, it may select sub-optimal arms too often which will also increase its cumulative regret. There are a variety of algorithms that solve this exploration/exploitation trade-off problem \citep{auer2002finite, auer2002using, auer2002nonstochastic, agrawal2012analysis, bubeck2012regret}.        

The contextual bandit problem extends the classical MAB setting, with the addition of time-varying side information, or {\em context}, made available at every time step. 
%Contexts could be available for every arm, or there could be just one context describing the state of the problem at that time. 
The best arm at every time step depends on the context, and  intuitively the learner seeks to determine the best arm as a function of context. To date, work on contextual bandits has studied cumulative regret minimization, which is motivated by applications in health care, web advertisement recommendations and news article recommendations \citep{li2010contextual}. The contextual bandit setting is also called associative reinforcement learning \cite{auer2002using} and linear bandits \citep{agrawal2012analysis,abbasi2011improved}. 

In classical (non-contextual) MABs, the goal of the learner isn't always to minimize the cumulative regret. In some applications, there is a {\em pure exploration} phase during which the learning incurs no regret (i.e., no penalty for sub-optimal decisions), and performance is measured in terms of {\em simple regret}, which is the regret assessed at the end of the pure exploration phase. For example, in top-arm identification, the learner must guess the arm with highest expected reward at the end of the exploration phase. Simple regret minimization clearly motivates different strategies, since there is no penalty for sub-optimal decisions during the exploration phase. Fixed budget and fixed confidence are the two main theoretical frameworks in which simple regret is generally analyzed \cite{gabillon2012best,jamieson2014best,garivier2016optimal,carpentier2015simple}. 

In this paper, we extend the idea of simple regret minimization to contextual bandits. In this setting, there is a pure exploration phase during which no regret is incurred, following by a {\em pure exploitation} phase during which regret is incurred, but there is no feedback so the learner cannot update its policy. To our knowledge, previous work has not addressed novel algorithms for this setting.
\citet{guan2018nonparametric} provide simple regret guarantees for the policy of uniform sampling of arms in the i.i.d setting. The contextual bandit algorithm of \citet{tekin2015releaf} also has distinct exploration and exploitation phases, but unlike our setting, the agent has control over which phase it is in, i.e., when it wants to receive feedback.
In the work of \citet{hoffman2014correlation, soare2014best,libin2017bayesian, xu2017fully} there is a single best arm even when contexts are observed (directly or indirectly).
Our algorithm, Contextual-Gap, generalizes the idea of Bayes Gap algorithm \cite{hoffman2014correlation} and UGapEb algorithm \cite{gabillon2012best} to the contextual bandits setting. Theoretical analysis of bounding Simple regret for contextual bandits is one of the key contributions of the paper. Specifically, proof of Lemma \ref{lem:LowestEigenValueSelfAdjoint} using Freedman’s inequality for self-adjoint operators and proof of Lemma \ref{lem:UCB_Bound_Self_Adjoint_Rankr} which gives monotonic upper bound on confidence widths makes the analysis interesting and novel. 

%Also, note that with sufficient sampling of contexts and arms, contextual bandits even in the cumulative regret setting provide \(\epsilon\) optimality. 

We make the following contributions: 1. We formulate a novel problem: that of simple regret minimization for contextual bandits. 2. We develop an algorithm, Contextual-Gap, for this setting. 3. We present performance guarantees on the simple regret in the fixed budget framework. 4. We present experimental results for synthetic dataset, and for adaptive sensor selection in nano-satellites.

The paper is organized as follows. In section 2, we motivate the new problem based on the real-life application of magnetometer selection in spacecraft. In section 3, we state the problem formally, and to solve this new problem, we present the Contextual-Gap algorithm in section 4. In section 5, we present the learning theoretic analysis and in section 6, we present and discuss experimental results. Section 7 concludes.

%% file: 2_motivation.tex
%-------------------------------------------------------------------------------------%
\section{Motivation}
\label{sec:motivation}
%-------------------------------------------------------------------------------------%

Our work is motivated by autonomous systems that go through an initial training phase (the pure exploration phase) where they learn how to accomplish a task without being penalized for sub-optimal decisions, and then are deployed in an environment where they no longer receive feedback, but regret is incurred (the pure exploitation phase). 

An example scenario arises in the problem of estimating weak interplanetary magnetic fields in the presence of noise using resource-constrained spacecraft known as nano-satellites or CubeSats. Spacecraft systems generate their own spatially localized magnetic field noise due to large numbers of time-varying current paths in the spacecraft. Historically, with large spacecraft, such noise was minimized by physically separating the sensor from the spacecraft using a rigid boom. In highly resource-constrained satellites such as nano-satellites, however, structural constraints limit the use of long rigid booms, requiring sensors to be close to or inside the CubeSat (Figure \ref{fig:distmag}). Thus, recent work has focused on nano-satellites equipped with multiple magnetic field sensors (magnetometers), shorter booms and adaptive methods for reducing spacecraft noise \citep{sheinker2016adaptive} . 

A natural problem that arises in nano-satellites with multiple sensors is that of determining the sensor with the reading closest to the true magnetic field. At each time step, whenever sensor is selected, one has to calibrate the sensor readings because sensor behaviours change due to rapid changes in temperature and movement (rotations or maneuvers) of the satellite which introduce a further errors in magnetic field readings. This calibration process is expensive in terms of computation and memory \citep{kepko1996accurate, leinweber2012inflight}, particularly when dealing with many magnetic field sensors. To get accurate readings from different sensors one has to repeat this calibration process for every sensor and it's not feasible because of the limited power resources on the spacecraft. These constraints motivate the selection of a single sensor at each time step.

Furthermore, the best sensor changes with time. This stems from the time-varying localization of noise in the spacecraft, which in turn results from different operational events such as data transmission, spacecraft maneuvers, and power generation.  This dynamic sensor selection problem is readily cast as a contextual bandit problem. The context is given by the spacecraft's telemetry system which provides real-time measurements related to spacecraft operation, including solar panel currents, temperatures, momentum wheel information, and real-time current consumption \citep{springmann2012attitude}. 

In this application, however, conventional contextual bandit algorithms are not applicable because feedback is not always available. Feedback requires knowledge of sensor noise, which in turn requires knowledge of the true magnetic field. Yet the true magnetic field is known only during certain portions of a spacecraft's orbit (e.g., when the satellite is near other spacecraft, or when the earth shields the satellite from sun-induced magnetic fields). Moreover, when the true magnetic field is known, there is no need to estimate the magnetic field in the first place! This suggests a learning scenario where the agent (the sensor scheduler) operates in two phases, one where it has feedback but incurs no regrets (because the field being estimated is known), and another where it does not receive feedback, but nonetheless needs to produce estimates. This is precisely the problem we study.

In the magnetometer problem defined above, the exploration and exploitation times occur in phases, as the satellite moves into and out of regions where the true magnetic field is known. For simplicity, we will the address the problem in which the first \(T\) time steps belong to the exploration phase, and all subsequent time steps to the exploitation phase. Nonetheless, the algorithm we introduce can switch between phases indefinitely, and does not need to know in advance when a new phase is beginning.

\begin{figure}[htb]
\centering
        \includegraphics[width=0.45\textwidth]{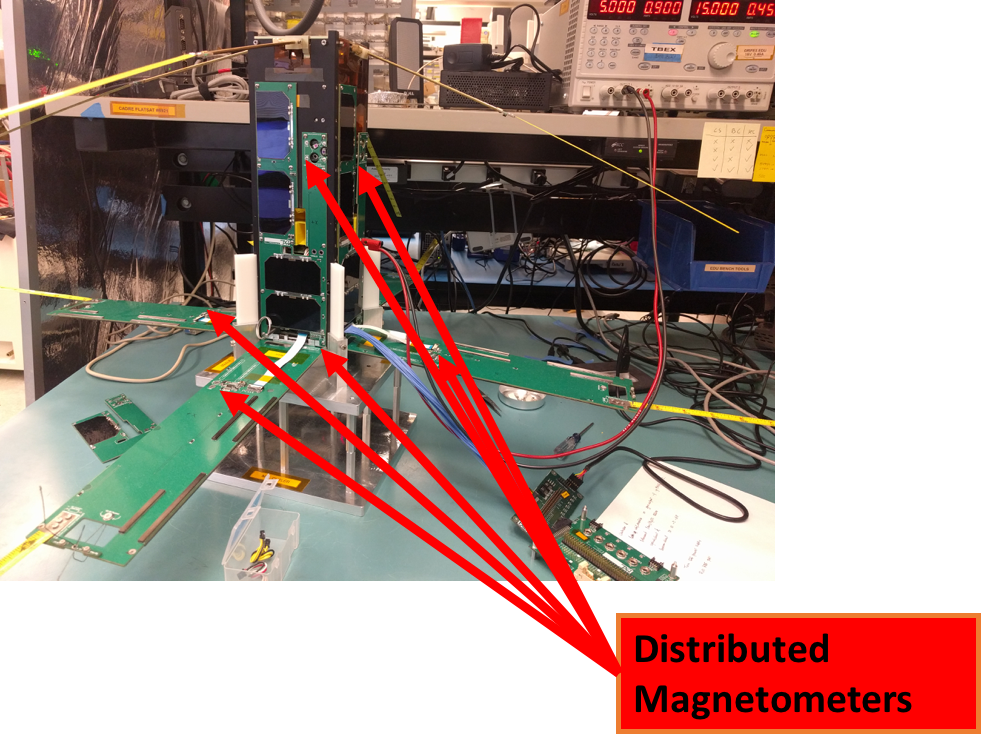}
        \caption{TBEx Small Satellite with Multiple Magnetometers \citep{tsunoda2016tilts, nathanael2018tbex}}
        \label{fig:distmag}
\end{figure}

Sensor management, adaptive sensing, and sequential resource allocation have historically been viewed in the decision process framework where the learner takes actions on selecting the sensor based on previously collected data. There have been many proposed solutions based on Markov decision processes (MDPs) and partially observable MDPs, with optimality bounds for cumulative regret \citep{hero2011sensor, castanon1997approximate, evans2001optimal, krishnamurthy2002algorithms, chong2009partially}. In fact, sensor management and sequential resource allocation was one of the original motivating settings for the classical MAB problem \citep{mahajan2008multi, bubeck2012regret, hero2011sensor}, again with the goal of cumulative regret minimization. We are interested in an adaptive sensing setting where the optimal decisions and rewards also depend on the context, but where the actions can be separated into a pure exploration and pure exploitation phases, with no regret during exploration, and with no feedback during pure exploitation.

%% file: 3_formal_setting.tex
%-------------------------------------------------------------------------------------%
\section{Formal Setting}
\label{sec:formal_setting}
%-------------------------------------------------------------------------------------%
We denote the context space as \(\sX = \RR^{d} \). Let \(\{x_t\}_{t=1}^{\infty}\) denote the sequence of observed contexts. Let the total number of arms be \(A\). For each \(x_t\), the learner is required to choose an arm \(a \in [A]\), where \([A] := \{1,2,...,A\}\). 

For arm \( a \in [A]\), let \(f_a: \sX \rightarrow \Rbb\) be a function that maps context to expected reward when arm $a$ is selected. Let \(a_t\) denote the arm selected at time \( t \), and assume the reward at time $t$ obeys \(r_{t}:= f_{a_t}(x_t) + \zeta_t\), where $\zeta_t$ is noise (described in more detail below). We assume that for each $a$, $f_a$ belongs to a reproducing kernel Hilbert space (RKHS) defined on \(\sX\). 
%let \(\{\zeta_t\}_{t=1}^{\infty}\) be a sequence of random variables,
%We consider functions over a reproducing kernel Hilbert space (RKHS) defined on \(\sX\)
The first \(T\) time steps belong to the \emph{exploration phase} where the learner observes context \(x_t\), chooses arm \(a_t\) and obtains reward \(r_{t}\). The time steps after \(T\) belong to an \emph{exploitation phase} where the learner observes context \(x_t\), chooses arm \(a_t\) and earns an implicit reward \(r_{t}\) that is not returned to the learner.

For the theoretical results below, the following general probabilistic framework is adopted, following \citet{abbasi2011improved} and \citet{durand2017streaming}. %Assume \(\sX\) is a compact space endowed with a finite positive Borel measure. 
%A filtration is a sequence of \(\sigma\)-algebras \(\{\mathcal{F}_t\}_{t=1}^{\infty}\) such that \(\mathcal{F}_1 \subseteq \mathcal{F}_2 \subseteq \cdots \subseteq \mathcal{F}_n \subseteq \cdots \). Let \(\{\mathcal{F}_t\}_{t=1}^{\infty}\) be a filtration such that $x_t$ is $\mathcal{F}_{t-1}$ measurable, and $\zeta_t$ is $\mathcal{F}_{t}$ measurable. For example, one may take \(\mathcal{F}_t:= \sigma(x_1, x_2, \cdots, x_{t+1}, \zeta_1, \zeta_2, \cdots, \zeta_t)\), i.e., \( \mathcal{F}_t \) is the \( \sigma-\)algebra generated by  \( x_1, x_2, \cdots, x_{t+1}, \zeta_1, \zeta_2, \cdots, \zeta_t\).  
We assume that \(\zeta_t\) is a zero mean, \(\rho\)-conditionally sub-Gaussian random variable, i.e., \(\zeta_t\) is such that for some \(\rho > 0\) and \(\forall \gamma \in \RR \),
\begin{equation}
\EE[e^{\gamma\zeta_t}|\mathcal{H}_{t-1}] \leq \exp\bigg(\frac{\gamma^2 \rho^2}{2} \bigg).
\end{equation}
Here \( \mathcal{H}_{t-1} = \{ x_1, \dots, x_{t-1}, \zeta_1, \dots, \zeta_{t-1}\} \) is the history at time \( t \) (see supplementary material for additional details).

We also define the following terms. Let \( D_{a,t} \) be the set of all  time indices when arm \( a \) was selected up to time \( t -1 \) and set \(  N_{a,t} = |D_{a,t}|\).  Let \(X_{a,t}\) be the data matrix whose columns are \( \{ x_{\tau} \}_{\tau \in D_{a,t}} \)  and similarly let \(Y_{a,t} \) denote the column vector of rewards \( \{ r_{\tau}
%(x_{\tau}) 
\}_{\tau \in D_{a,t}} \). Thus, \(X_{a,t} \in \mathbb{R}^{d \times N_{a,t}}\) and \( Y_{a,t} \in \mathbb{R}^{N_{a,t}} \). 

\subsection{Problem Statement}
At every time step \( t\), the learner observes context \( x_t \). During the exploration phase \(t \le T\), the learner chooses a series of actions to explore and learn the mappings \(f_a\) from context to reward. During the exploitation phase  \( t > T\), the goal is to select the best arm as a function of context. We define the {\em simple regret} associated with choosing arm \(a \in [A]\), given context \(x\), as:
\begin{equation}\label{eq:SimpleRegretDefinition}
R_{a}(x) := f^*(x) - f_a(x),
\end{equation}
where \( f^*(x) := \max_{i \in [A] } f_{i}(x)\) is the expected reward for the best arm for context \(x\). The learner aims to minimize the simple regret for \( t > T \). To be more precise, let $\Omega$ be the fixed policy mapping context to arm during the exploitation phase. The goal is to determine policies for the exploration and exploitation phases such that for all $\epsilon > 0$ and $t > T$
$$
\PP (R_{\Omega(x_t)}(x_t) \ge \epsilon |x_t) \le b_\epsilon(T),
$$
where $b_\epsilon(T)$ is an expression that decreases to 0 as $T \to \infty$. The following section presents an algorithm to solve this problem.

%% file: 4_algorithm.tex
%-------------------------------------------------------------------------------------%
\section{Algorithm}
\label{sec:algorithm}
%-------------------------------------------------------------------------------------%

We propose an algorithm that extends the Bayes Gap algorithm \citep{hoffman2014correlation} to the contextual setting. Note that Bayes Gap itself is originally motivated from UGapEb \citep{gabillon2012best}.

\subsection{Estimating Expected Rewards}

A key ingredient of our extension is an estimate of \( f_a\), for each \( a \), based on the current history. We use kernel methods to estimate \( f_a\).
Let \(k: \sX \times \sX \rightarrow \Rbb\) be a symmetric positive definite kernel function on \(\sX\), \(\sH\) be the corresponding RKHS and \(\phi(x) = k(\cdot, x)\) be the associated canonical feature map. Let \(\phi(X_{a,t}) := [\phi(x_j)]_{j\in D_{a,t}} \). We define the kernel matrix associated with \(X_{a,t}\) as \(K_{a,t}:= \phi(X_{a,t})^T \phi(X_{a,t}) \in \Rbb^{N_{a,t} \times N_{a,t}} \) and the kernel vector of context \(x\) as \(k_{a,t}(x):= \phi(X_{a,t})^T \phi(x)\). Let \(I_{a,t}\) be the identity matrix of size \(N_{a,t}\). 
%We work with the assumption that \(f_{a}(x) = \phi(x)^T \theta_a\) for some \(\theta_a \in \sH\). 
We estimate \(f_a\) at time \( t \), via kernel ridge regression, i.e., \[ \hat{f}_{a,t}(x) = \argmin_{f_a \in \sH } \sum_{j \in D_{a,t}} (f_a(x_j) - r_j )^2 + \lambda \| f_a \|^2. \] The solution to this optimization problem is \(\hat{f}_{a,t}(x) = k_{a,t}(x)^T(K_{a,t} + \lambda I_{a,t})^{-1}Y_{a,t} \). 
Furthermore, \citet{durand2017streaming} establish a confidence interval for \(f_{a}(x)\) in terms of \(\hat{f}_{a,t}(x)\) and the ``variance'' \(\hat{\sigma}_{a,t}^2(x):= k(x,x) - k_{a,t}(x)^T(K_{a,t} + \lambda I_{a,t})^{-1}k_{a,t}(x)\).

\begin{theorem}[Restatement of Theorem 2.1 in \\ \cite{durand2017streaming}] \label{Thm:MartingaleKernelUCB}
Consider the contextual bandit scenario described in section \ref{sec:formal_setting}.
For any \(\beta > 0\), with probability at least \( 1 - e^{-\beta^2} \), it holds simultaneously over all \( x \in \sX \) and all \( t \leq T \),  
\begin{equation}
|f_a(x) - \hat{f}_{a,t}(x)| \leq ( C_1\beta + C_2 ) \frac{\hat{\sigma}_{a,t}(x)}{\sqrt{\lambda}},
\end{equation}
 where
 \(C_1 = \rho \sqrt{2}\) and \\
 \(C_2 = \rho \sqrt{\sum_{\tau=2}^T \ln(1+ \frac{1}{\lambda} \hat{\sigma}^2_{a,\tau-1}(x_\tau) )} + \sqrt{\lambda} \n{f_a}_\sH.\) 
\end{theorem}
In the supplementary material we show that \(C_2 = O(\rho \sqrt{\ln{T}})\). For convenience, we denote the width of the confidence interval \( s_{a,t}(x) := 2(C_1 \beta + C_2 ) \frac{\hat{\sigma}_{a,t}(x)}{\sqrt{\lambda}}\). Thus, the upper and lower confidence bounds of \(f_a(x)\) are \(U_{a,t}(x) := \hat{f}_{a,t}(x) + \frac{s_{a,t}(x)}{2}\) and \( L_{a,t}(x) := \hat{f}_{a,t}(x) - \frac{s_{a,t}(x)}{2}\). The upper confidence bound is the most optimistic estimate of the reward and the lower confidence bound is the most pessimistic estimate of the reward.

\subsection{Contextual-Gap Algorithm}
\begin{algorithm}[h]
\caption{Contextual-Gap}
\label{alg:contextualGap}
\begin{algorithmic}
\STATE \textbf{Input:} Number of arms \( A \), Time Steps \( T \), parameter \(\beta\), regularization parameter \(\lambda \), burn-in phase constant \( N_{\lambda}\).
\STATE{//  Exploration Phase I: Burn-in Period // }
\FOR {\( t = 1,...,AN_{\lambda} \)}
\STATE Observe \(x_t\), choose \(  a_t = t \mod A \) \& receive \( r_{t} \in  \mathbb{R}\)
%\STATE Choose \(  a_t = t \mod A \)% \blue{(Where\(\mod \)is the remainder operation.)} 
%\STATE Receive reward \( r_{t} \in  \mathbb{R}\)
%\STATE Try every arm \( N_{\lambda} \)times
\ENDFOR
\STATE{//Exploration Phase II: Contextual-Gap Policy // }
\FOR {\( t = AN_{\lambda}+1,\dots,T \)}
\STATE Observe context \(x_t\)
\STATE Learn reward estimators \( \hat{f}_{a,t}(x_t) \) and confidence interval \( s_{a,t}(x_t) \) based on history
\STATE  \( U_{a,t}(x_t) =   \hat{f}_{a,t}(x_t) + \frac{s_{a,t}(x_t)}{2} \)
\STATE \( L_{a,t}(x_t) =   \hat{f}_{a,t}(x_t) - \frac{s_{a,t}(x_t)}{2} \)
%\STATE \( B_{a,t}(x_t) =  L_{a,t}(x_t) - \max_{i \neq a} U_{i,t}(x_t) \)
\STATE \( B_{a,t}(x_t) = \max_{i \neq a} U_{i,t}(x_t) -  L_{a,t}(x_t) \)
%\State \( J_{t}(x_t) = \argmax_a B_{a,t}(x_t) \)
\STATE \( J_{t}(x_t) = \argmin_a B_{a,t}(x_t) \)
\STATE \( j_{t}(x_t) = \argmax_{a \neq J_{t}(x_t)} U_{a,t}(x_t) \)
\STATE Choose \( a_t = \argmax_{a \in \{j_{t}(x_t),J_{t}(x_t)\}} s_{a,t}(x_t)\)
\STATE Receive reward \( r_{t}  \in  \mathbb{R}\)
\ENDFOR
\STATE{//  Exploitation Phase // }
\FOR {\(t>T\)}
	\STATE Observe context \(x_t\).
	\FOR {\(\tau = AN_{\lambda}+1,\dots,T\)}
    	\STATE Evaluate and collect \(J_{\tau}(x_{t}), B_{J_{\tau}(x_{t})}(x_t)\)
    \ENDFOR
    \STATE  \( \iota = \argmin_{AN_{\lambda}+1 \leq \tau \leq T} B_{J_{\tau}(x_{t}), t}(x_t)\)
    \STATE Choose \(\Omega(x_t) = J_{\iota}(x_t) \).
\ENDFOR
\end{algorithmic}
\end{algorithm}

During the exploration phase, the Contextual-Gap algorithm proceeds as follows. First, the algorithm has a burn-in period where it cycles through the arms (ignoring context) and pulls each one $N_\lambda$ times. Following this burn-in phase, when the algorithm is presented with context $x$ at time $t \le T$, the algorithm identifies two candidate arms, $J_t(x)$ and $j_t(x)$, as follows. For each arm $a$ the {\em contextual gap} is defined as \(B_{a,t}(x) :=      \max_{i \neq a} U_{i,t}(x) - L_{a,t}(x) \). \( J_t(x) \) is the arm that {\em minimizes} $B_{a,t}(x)$ and $j_t(x)$ is the arm (excluding \( J_t(x)\)) whose upper confidence bound is maximized. Among these two candidates, the one with the widest confidence interval is selected.  Quantity \( B_{a,t}(x) \) is used to bound the the simple regret for corresponding arm \(a \) and is the basis of definition of the best arm \( J_t(x)\). We use \(j_t(x) = \argmax_{a \neq J_{t}(x)} U_{a,t}(x) \) as the second candidate because optimistically \( j_t(x) \) has a chance to be the best arm and it may give more information about how bad the choice of \( J_t(x)\) could be.

In the exploitation phase, for a given context \(x\), the contextual gap for all time steps in the exploration phase are evaluated. The arm with the smallest gap over the entire exploration phase for the given context \(x\) is chosen as the best arm associated with context \(x\). 
Intuitively, this selection is because the arm with the smallest gap over the entire exploration phase represents the arm with the most certainty over the quality of its reward estimates over the other arms.
Because there is no feedback during the exploitation phase, the algorithm moves to the next exploitation step without modification to the learning history. The exact description is presented in Algorithm \ref{alg:contextualGap}.

During the exploitation phase, looking back at all history may be computationally prohibitive. Thus, in practice, we just select the best arm as \( J_{T}(x_t), \forall t > T \). As described in the experimental section, this works well in practice. Theoretically, \( N_{\lambda} \) has to be bigger than a certain number defined in Lemma \ref{lem:UCB_Bound_Self_Adjoint_Rankr}, but for experimental results we keep \( N_{\lambda} = 1\).

\subsection{Comparison of Contextual-Gap and Kernel-UCB}
In this section, we illustrate the difference between the policies of Kernel-UCB (which minimizes cumulative regret) and exploration phase of Contextual-Gap (which aims to minimize simple regret). 
At each time step, Contextual-Gap selects one of two arms: $J_t(x)$, the arm with highest pessimistic reward estimate, or $j_t(x)$, the arm excluding $J_t(x)$ with highest optimistic reward estimate.

Kernel-UCB, in contrast, selects the arm with the highest optimistic reward estimate (i.e., with the maximum upper confidence bound). 

\begin{figure}[htb]
\centering
        \includegraphics[width=0.35\textwidth]{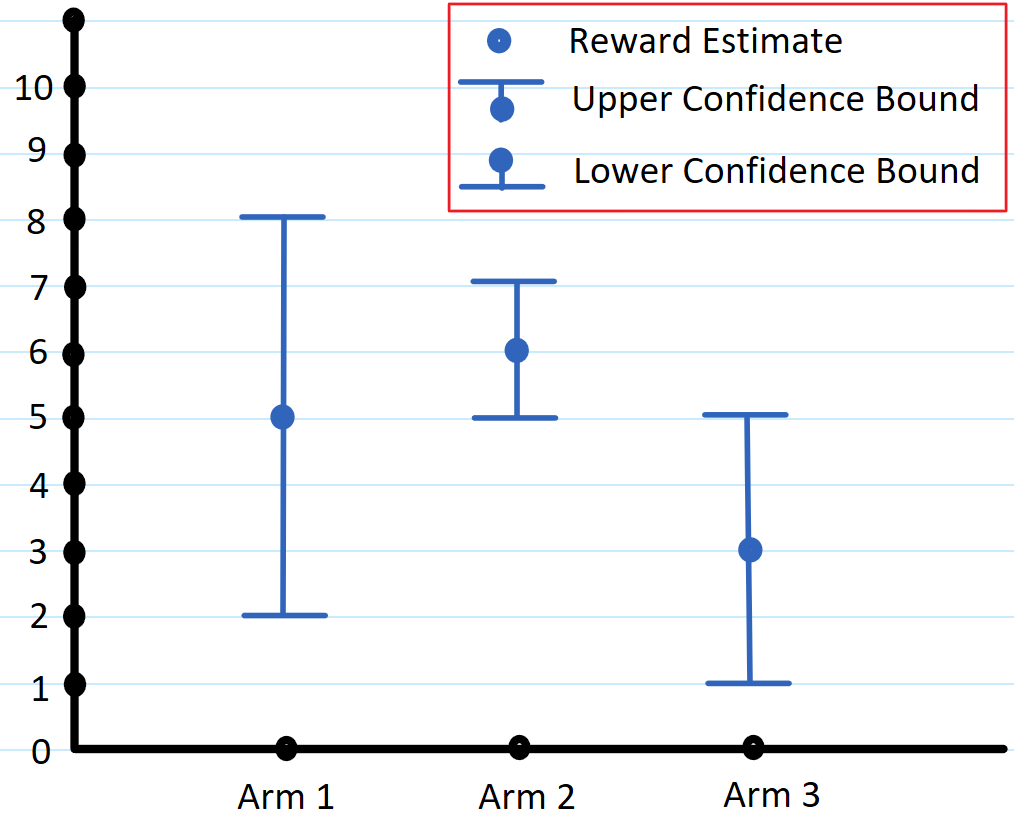}
        \caption{case 1}
        \label{fig:case1}
\end{figure}
Consider a three arm scenario at some time \( \tau \) with context \( x_{\tau} \). Suppose that the estimated rewards and confidence intervals are as in Figures \ref{fig:case1} and \ref{fig:case2}, reflecting two different cases. 
%Consider a three arm scenario at some time \( \tau \). We observe the context \( x_{\tau} \). Based on the history of the data, we estimate the reward for the context \( x_{\tau} \) for all 3 arms, which is shown by a blue dot in the Figure \ref{fig:case1} and Figure \ref{fig:case2}. We also calculate upper and lower confidence bounds for each of these reward estimates.
\begin{itemize}[topsep=0pt, noitemsep]
    \item Case 1 (Figure \ref{fig:case1}): In this case, Kernel-UCB would pick arm 1, because it has the maximum upper confidence bound. Kernel-UCB's policy is designed to be optimistic in the case of uncertainty. In the Contextual-Gap, we first calculate \( J_{\tau}(x_{\tau}) \) which minimizes \(B_{a,\tau}(x_{\tau}) \). Note that  \(B_{1,\tau}(x_{\tau}) =  U_{2,\tau}(x_{\tau}) - L_{1,\tau}(x_{\tau})  = 7 - 2  = 5 \),     \(B_{2,\tau}(x_{\tau}) = 3 \) and   \(B_{3,\tau}(x_{\tau}) = 7 \). In this case, \( J_{\tau}(x_{\tau}) = 2 \) and hence \( j_{\tau}(x_{\tau}) = 1\). Finally, Contextual-Gap would choose among arm 1 and arm 2, and would finally choose arm 1 because it has the largest confidence interval. Hence, in case 1, Contextual-Gap chooses the same arm as that of Kernel-UCB. 
    \item  Case 2 (Figure \ref{fig:case2}): In this case, Kernel-UCB would pick arm 1. Note that  \(B_{1,\tau}(x_{\tau}) = U_{2,\tau}(x_{\tau}) -  L_{1,\tau}(x_{\tau}) = 7 - 4  = 3 \),     \(B_{2,\tau}(x_{\tau}) = 7 \) and   \(B_{3,\tau}(x_{\tau}) = 4 \). Then \( J_{\tau}(x_{\tau}) = 1 \) and hence \( j_{\tau}(x_{\tau}) = 2\). Finally,  Contextual-Gap chooses arm 2, because it has the widest confidence interval. Hence, in case 2, Contextual-Gap chooses a different arm compared to that of Kernel-UCB. 
\end{itemize}
\begin{figure}[htb]
\centering
        \includegraphics[width=0.35\textwidth]{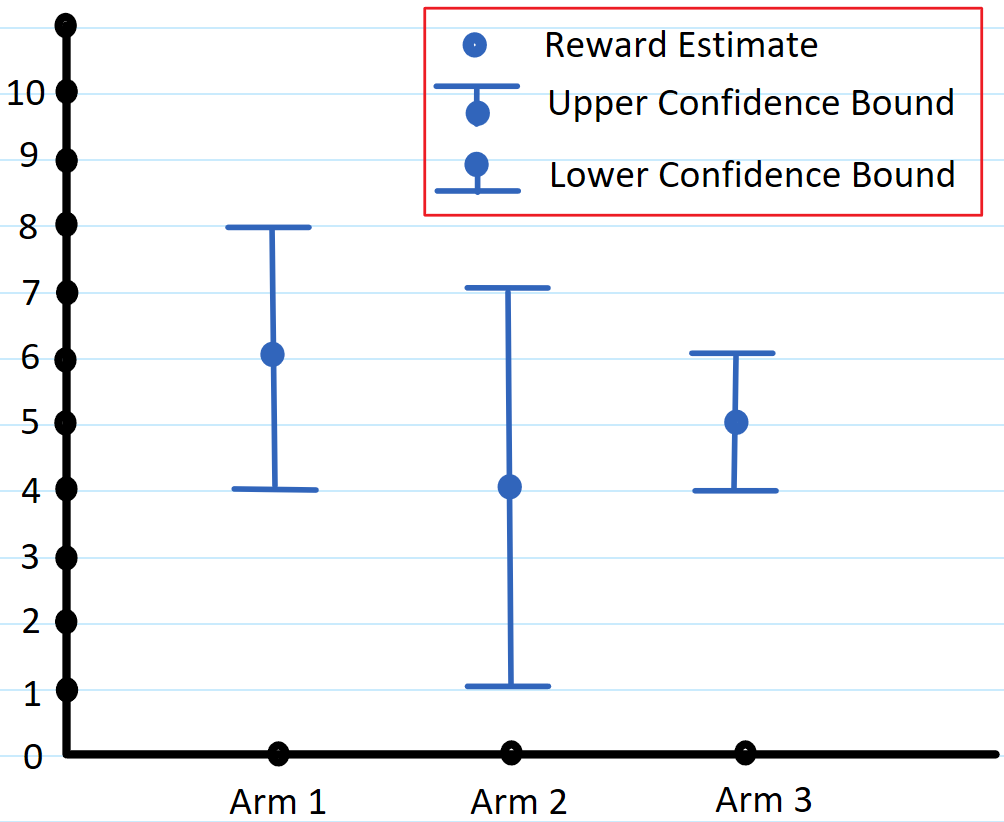}
        \caption{case 2}
        \label{fig:case2}
\end{figure}
Clearly, the use of the lower confidence bound along with upper confidence bound allows Contextual-Gap to explore more than kernel-UCB. However, Contextual-Gap doesn't explore just any arm, but rather it explores only among arms with some likelihood of being optimal. The following section details high probability bounds on the simple regret of the Contextual-Gap algorithm.

%% file: 5_learning_theoretic_analysis.tex
%-------------------------------------------------------------------------------------%
\section{Learning Theoretic Analysis}
\label{sec:learning_theoretic_analysis}
%-------------------------------------------------------------------------------------%
We now analyze high probability simple regret bounds which depend on the gap quantity \(\Delta_a(x):= | \max_{i \neq a} f_i(x) - f_a(x)|\). The bounds are presented in the non-i.i.d setting described in Section \ref{sec:formal_setting}.  For the confidence interval to be useful, it needs to shrink to zero with high probability over the feature space as each arm is pulled more and more. This requires the smallest non-zero eigenvalue of the sample covariance matrix of the data for each arm to be lower bounded by a certain value. 
We make an assumption that allows for such a lower bound, and use it to prove that the confidence intervals shrink with high probability under certain assumptions. Finally, we bound the simple regret using the result of shrinking confidence interval, the gap quantity, and the special exploration strategy described in Algorithm \ref{alg:contextualGap}. We now make additional assumptions to the problem setting.

\begin{enumerate}[align=left, leftmargin=*, label=\textbf{A \Roman*}, topsep=0pt, noitemsep]
\item \label{AsX} \(\{ \sX_t \}_{t \geq 1} \subset \Rbb^d\), is a random process on compact space endowed with a finite positive Borel measure.
\item \label{Akernel} Kernel \(k: \sX \times \sX \rightarrow \Rbb\) is bounded by a constant \( L \), the canonical feature map \(\phi: \sX \rightarrow \sH \) of \(k\) is a continuous function, and \( \sH \) is separable.%  and \(f_a \in \sH\). %Without loss of generality, we assume that \( L = 1 \).
\end{enumerate}

We denote \(\EE_{t-1}[\cdot]:= \EE[ \cdot | x_1, x_2, \ldots, x_{t-1}] \) and by \(\lambda_r(A)\) the \(r^{\textnormal{th}}\) largest eigenvalue of a compact self adjoint operator \(A\). For a context \(x\), the operator \(\phi(x)\phi(x)^T: \sH \rightarrow \sH\) is a compact self-adjoint operator. %We define the cumulative operator \(V_t:= \sum_{s=1}^{t-1} \EE_s[\phi(x_t)\phi(x_t)^T]\). 
Based on this notation, we make the following assumption:

\begin{enumerate}[align=left, leftmargin=*, label=\textbf{A \Roman*}, topsep=0pt, noitemsep]
\setcounter{enumi}{2}
\item \label{Alambda} There exists a subspace of dimension \(d^\ast\) with projection \(P\), and a constant \(\lambda_{x} > 0\), such that \( \forall t \), \(\lambda_r(P^T\EE_{t-1}[\phi(x_t)\phi(x_t)^T]P) > \lambda_x\) for \(r \leq d^\ast\) and \(\lambda_r((I-P)^T\EE_{t-1}[\phi(x_t)\phi(x_t)^T](I-P)) = 0, \forall r > d^\ast\).
\end{enumerate}
Assumption \ref{Alambda} facilitates the generalization of Bayes gap \citep{hoffman2014correlation} to the kernel setting with non-i.i.d, time varying contexts.
It allows us to lower bound, with high probability, the \(r^{\textnormal{th}}\) eigenvalue of the cumulative second moment operator \(S_t:= \sum_{s=1}^t \phi(x_s) \phi(x_s)^T\) so that it is possible to learn the reward behavior in the low energy directions of the context at the same rate as the high energy ones with high probability.

We now provide a lower bound on the \(r^{th}\) eigenvalue of a compact self-adjoint operator. There are similar results in the setting where reward is a linear function of context, including Lemma 2 in \citet{gentile2014online} and Lemma 7 in \citet{li2017online} which provides lowest eigenvalue bounds with the assumption of linear reward and full rank covariance, and Theorem 2.2 in \citet{tu2017least} which assumes more structure to the contexts generated. There are results similar to Lemma 2 in  \citet{gentile2014online}, \citet{gentile2017context} and \citet{korda2016distributed}. We extend these results to the setting of a compact self-adjoint operator scenario with data occupying a finite dimensional subspace. 

Let \(W_t := \sum_{s=1}^t \EE_{s-1}[ (\phi(x_s) \phi(x_s)^T)^2 ] - (\EE_{s-1}[\phi(x_s) \phi(x_s)^T])^2\). By construction and Assumption \ref{Alambda} we can show that \(W_t\) has \(d^\ast\) non-zero eigenvalues (See Section 4.1 in the supplementary material). %\blue{Dr. Scott: Why is \(W_t\) compact?} product is compact and then sum is compact hence W_t is compact. 
%
%and we will denote by \(\lambda_{r}(W_t)\) the \(r^{\textnormal{th}}\) largest eigenvalue of \(W_t\)
%\blue{We need to discuss why this assumptions make sense. Give some more setting where it would work and where it may not. Should we also add assumption of \( N_{\lambda}  \) here?}
\begin{lemma}[Lower bound on \( r^{th} \) 
Eigen-value of compact self-adjoint operators] \label{lem:LowestEigenValueSelfAdjoint}
Let \( x_t \in \sX \), \( t \geq 1 \) be generated sequentially from a random process. Assume that conditions \ref{AsX}-\ref{Alambda} hold. Let \( p(t) = \min(-t,1)\) and \( \forall b \geq 0, a > \frac{1}{6}(L^2 + \sqrt{L^4 + 36b})\) let \(\tilde{d}:= 50 \sum_{r=1}^{d^\ast} p(-\frac{a\lambda_r(\EE W_t)}{L^2 b}) \leq 50 d^\ast \). Let 
\[ A(t,\delta) = \log \frac{(tL^4 +1)(tL^4 +3)\tilde{d}}{\delta},\]
and 
\[ h(t, \delta) = \Big(t \lambda_{x} - \frac{L^2}{3}\sqrt{18t A(t,\delta) + A(t,\delta)^2}) - \frac{L^2}{3} A(t,\delta) \Big).\]
Then for any \( \delta > 0 \), 
\begin{equation*} 
\lambda_{r}(S_t) \geq  h(t, \delta)_+
\end{equation*}
holds for all \( t > 0 \) with probability at least \( 1 - \delta \). Furthermore, if L =1, \(r \leq d^\ast\) and \( 0 < \delta \leq \frac{1}{8}\), then the event 
\[  \lambda_{r}(S_t) \geq \frac{t \lambda_{x}}{2} , \forall t \geq \frac{256}{\lambda_{x}^2} \log (\frac{128\tilde{d}}{\lambda_{x}^2 \delta}), \]
holds with probability at least \( 1 - \delta \) .  
\end{lemma}

Lemma \ref{lem:LowestEigenValueSelfAdjoint} provides high probability lower bounds on the minimum nonzero eigenvalue of the cumulative second moment operator \(S_t\). Using the preceding lemma and the confidence interval defined in Theorem \ref{Thm:MartingaleKernelUCB}, it is possible to provide high probability monotonic bounds on the confidence interval widths \(s_{a,t}(x)\). 
 
\begin{lemma}[Monotonic upper bound of \(s_{a,t}(x_t) \) ] \label{lem:UCB_Bound_Self_Adjoint_Rankr}
Consider a contextual bandit simple regret minimization problem with assumptions \ref{AsX}-\ref{Alambda} and  fix \(T\). Assume \(\n{\phi(x)} \leq 1\), \( \lambda > 0 \) and \( \forall a \in [A] \), \( N_{a,t} > N_{\lambda} := \max \Big( \frac{2(1 -\lambda) }{\lambda_{x}} , d^\ast, \frac{256}{\lambda_{x}^2} \log (\frac{128\tilde{d}}{\lambda_{x}^2 \delta}) \Big) \). Then, for any \( 0 < \delta \leq \frac{1}{8} \),
\[ 
s_{a,t}(x_t)^2 \leq g_{a,t}(N_{a,t})
\]
with probability at least \( 1 - \delta \),
for the monotonically decreasing function \(g_{a,t}\) defined as \(g_{a,t}(N_{a,t}) := 8 (C_1 \beta + C_2)^2 \Big(   \frac{1}{\lambda+N_{a,t}\lambda_{x}/2}  \Big)\).
\end{lemma}

The condition \(N_{a,t} > N_\lambda\) results in a minimum number of tries that arm \(a\) has to be selected before any bound will hold. In \(N_\lambda := \max \Big( \frac{2(1 -\lambda) }{\lambda_{x}} , d^\ast, \frac{256}{\lambda_{x}^2} \log (\frac{128\tilde{d}}{\lambda_{x}^2 \delta}) \Big)\), the first and third term in the \( \max \) are needed so that we can give concentration bounds on eigenvalues and prove that the confidence width shrinks. The second term is needed because one has to get at least \( d^\ast \) contexts for every arm so that at least some energy is added to the lowest eigenvalues.

These high probability monotonic upper bounds on the confidence estimate can be used to upper bound the simple regret. The upper bound depends on a context-based hardness quantity defined for each arm \(a\) (similar to \citet{hoffman2014correlation}) as 
\begin{equation}
H_{a,\epsilon}(x) = \max(\frac{1}{2}(\Delta_a(x) + \epsilon), \epsilon). 
\end{equation}
%We define a feasible set \( A^{\prime}(x) \subseteq [A] \) such that elements of \( A^{\prime}(x) \) contain possible set of arms that may have been the best if context \( x \) was observed at all times \(t\) with the observed history until time \(t-1\), \(\{x_\tau\}_{tau=0}^{t-1}\). 
Denote its lowest value as \(H_{a,\epsilon}:= \inf_{x \in \sX} H_{a,\epsilon}(x) \). Let total hardness be defined as \(H_\epsilon := \sum_{a \in [A]} H_{a, \epsilon}^{-2}\) (Note that \(H_{\epsilon} \leq \frac{A}{\epsilon^2}\)). 
The recommended arm after time \(t \geq T\) is defined as \[\Omega(x) = J_{ \argmin_{A N_{\lambda} +1 \leq \tau \leq T} B_{J_{\tau}(x_t), t}(x_t)}(x_t)\] from Algorithm \ref{alg:contextualGap}. %For notational succinctness with subscripts, we denote \(\Omega(x)\) as \(\Omega\).
We now upper bound the simple regret as follows:
\begin{theorem}
\label{thm:MainTheorem}
Consider a contextual bandit problem as defined in Section \ref{sec:formal_setting} with assumptions \ref{AsX}-\ref{Alambda}. For  \( 0 < \delta \leq \frac{1}{8} \), \( \epsilon > 0 \) and \( N_{\lambda} := \max \Big( \frac{2(1 -\lambda) }{\lambda_{x}} , d^\ast, \frac{256}{\lambda_{x}^2} \log (\frac{128\tilde{d}}{\lambda_{x}^2 \delta}) \Big) \), %For arms \(a \leq A\) given the sequence \(\{x_\tau\}_{\tau=0}^{t-1}\), 
let 
\begin{equation}\label{eq:beta}
%\beta  =\sqrt{\frac{\lambda_{x}(T - N_\lambda(A - 1) - A) + A\lambda}{ 8 L C_1^2 H_{\epsilon} }} - \frac{C_2}{C_1},
\beta    = \sqrt{\frac{\lambda_{x}(T - N_{\lambda}( A -1)) + 2 A\lambda}{ 16  C_1^2 H_{\epsilon} }} - \frac{C_2}{C_1}.
\end{equation}
For all \(t > T\) and \(\epsilon  > 0\),
%For any context \(x \in \sX \) generated from the filtration and for \(t > T\), we can bound the simple regret for the recommended arm \(\Omega\) as 
\begin{equation}
\PP (R_{\Omega(x_t)}(x_t) < \epsilon | x_t) \geq 1 - A(T - AN_\lambda) e^{-\beta^2} - A\delta.
\end{equation}
\end{theorem}

Note that the term \( C_2 \) in \eqref{eq:beta} grows logarithmically in \( T\) (see supplementary material). For \( \beta \) to be positive, \( T \) should be greater than \(\frac{16 H_{\epsilon}C_2^2 -2 A\lambda}{\lambda_x} + N_{\lambda}(A-1) \). 
We compare the term \(e^{-\beta^2} \) in our bound with the uniform sampling technique in \citet{guan2018nonparametric} which leads to a bound that decay like \( Ce^{-c T^{\frac{2}{(d_1 + d)}}} \geq Ce^{-cT^{\frac{2}{(2+d)}}}\), where \( d_1 \geq 2 \), \( d \) is the context dimension, and \(C\) and \(c\) are constants. In our case, the decay rate has the form \( C^\prime Te^{-c^\prime T} \) for constants \(C^\prime, c^\prime\). Clearly, our bound is superior for \( \forall d \geq 1 \).
We can also compare Theorem \ref{thm:MainTheorem} with Bayes Gap  \citep{hoffman2014correlation} and UGapEb \citep{gabillon2012best} which provide simple regret guarantees in the multi-armed bandit setting. Bayes Gap and UGapEb have regret bounds of order \( O(AT e^{-\frac{T-A}{H_{\epsilon}}}) \) and we provide bounds of order \( O(A(T-AN_{\lambda}) e^{-\frac{T-AN_{\lambda}}{H_{\epsilon}}}) \). Ignoring other constants, our method has the additional term \( N_{\lambda}\) which is required because algorithm needs to see enough number of contexts to get information about context space and to become confident in the reward estimates in that context space.
The simple regret bound is also dependent on the gap between the arms. A larger gap quantity \( \Delta_a \) implies a larger \( H_{a, \epsilon}\) which implies that quantity \( e^{-\frac{1}{H_{\epsilon}}} \) is small. This means that a larger gap quantity leads to a faster rate.

Note that there are two choices for simple regret analysis: 1) bounding the simple regret uniformly (Theorem \ref{thm:MainTheorem}) and 2) average simple regret  \(\Big( \sum_{t >T} R_{\Omega(x_t)}(x_t) \Big)\). We bound the simple regret uniformly and it may require stronger distributional assumptions (e.g. Assumption \ref{Alambda} ) compared to average simple regret. 
%Furthermore, an analysis of average simple regret would require the length of the exploitation phase to tend to infinity (for a bound to shrink to zero) but, importantly, that is not the case for our analysis.
Furthermore, we provide uniform bounds and not average simple regret bounds since our problem setting of simple regret minimization and the motivating application require performance guarantees for every time step during exploitation, as opposed to average simple regret guarantees.

%% file: 6_results_discussion.tex
%-------------------------------------------------------------------------------------%
\section{Experimental Results and Discussion} 
\label{sec:results_and_discussion}
%-------------------------------------------------------------------------------------%
We present results from two different experimental setups, first is synthetic data, and second from a lab generated non-i.i.d spacecraft magnetic field as described in Section 2. We present average simple regret comparisons of the Contextual-Gap algorithm against five baselines: Uniform sampling, Epsilon Greedy, kernel-UCB (\cite{valko2013finite}), Kernel-UCB-Mod (Kernel UCB for exploration but best estimated reward for exploitation, Kernel-TS (Thompson Sampling) \cite{chowdhury2017kernelized}

\begin{figure*}[htb]
\centering
        \includegraphics[width=1.0\textwidth]{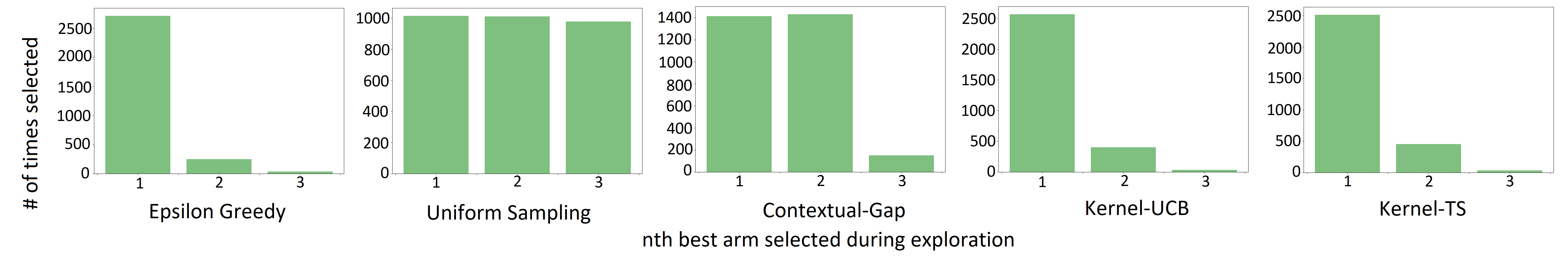}
        \caption{Histogram of Arm Selection during exploration}
        \label{fig:histogram_spacecraft_labexperimental}
\end{figure*}

For all the algorithms, we use the Gaussian kernel and tune the bandwidth of the kernel, and the regularization parameter (more details in supplementary material). The tuned parameters were used with the evaluation datasets to generate the plots. The code is available online to reproduce all results \footnote{The code to reproduce our results is available at \url{https://bit.ly/2Msgdfp}}.

The exploration parameter \( \alpha := C_1 \beta + C_2 \) is set to \(1\) for the results in this section. For spacecraft magnetic field dataset, we show results for different values of \( \alpha \) and worst case simple regret in the supplementary material. Note that one can not determine optimal $\alpha$ using theory because even if \( \lambda_x\) and \(N_\lambda \) are known, \(\alpha \) still needs to be tuned because of other constants involved and confidence bounds tend to be lose in practice. Also, Contextual-Gap is best no matter what \( \alpha \) is and it is not as sensitive to \( \alpha \) as other methods as shown in supplementary material.

\subsection{Synthetic Dataset}
We present results of contextual simple regret minimization for a synthetic dataset. At every time step, we observe a one dimensional feature vector \( x_t \sim \mathcal{U}[0, 2\pi] \), where \( \mathcal{U} \) is a uniform distribution. There are 20 arms and reward for each arm \(a \) is \( r_{a,t}  := \sin(a * x_t)\), where \( a = [1,2...,20] \).  The arm with the highest reward at time \( t \) is the best arm. At every time step, we only observe the reward for the arm that the algorithm selects. 

Since the dataset is i.i.d in nature, multiple simple regret evaluations are performed by shuffling the evaluation dataset, and the average curves are reported. Note that the algorithms have been cross validated for simple regret minimization. The plots are generated by varying the length of the exploration phase and keeping the exploitation dataset constant for evaluation of simple regret. It can be seen that the simple regret of the Contextual-Gap converges faster than the simple regret of other baselines. 

\begin{figure}[htb]
\vspace{-6pt}
\centering
        \includegraphics[width=0.44\textwidth]{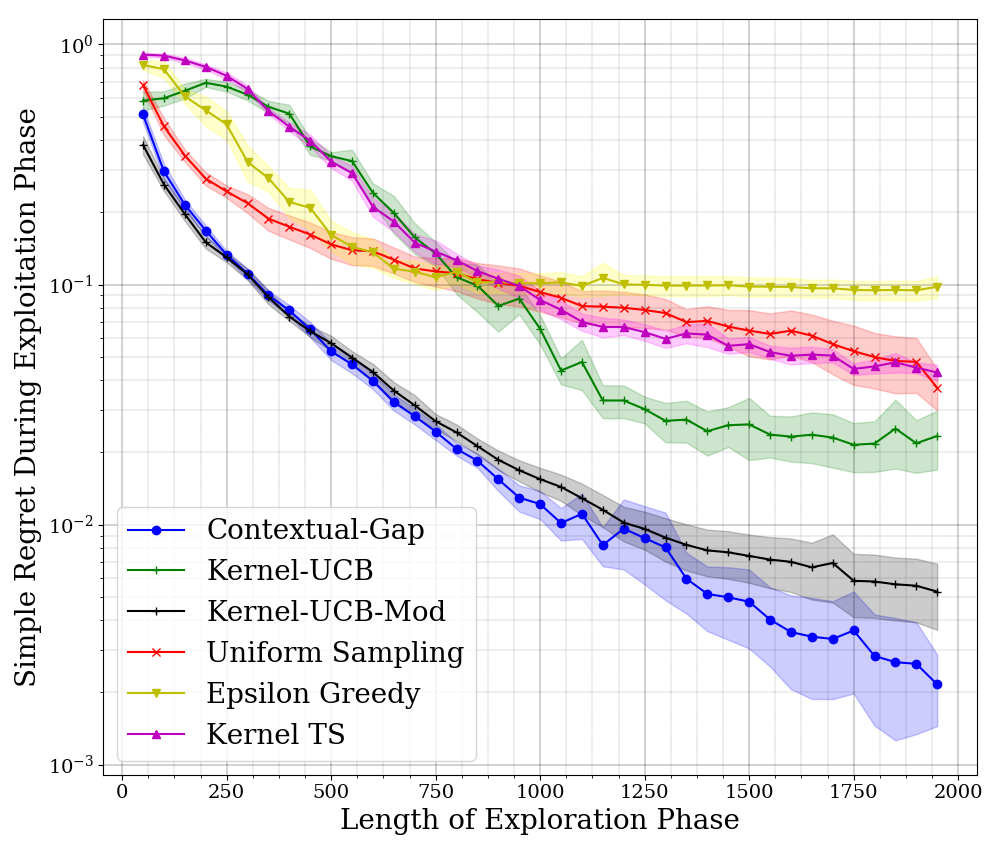}
        \caption{Avg. Simple Regret Evaluation on Synthetic Dataset}
        \label{fig:syntetic}
        \vspace{-6pt}
\end{figure}

\subsection{Experimental Spacecraft Magnetic Field Dataset}

We present the experimental setup and results associated with a lab generated, realistic spacecraft magnetic field dataset with \emph{non-i.i.d contexts}. In spacecraft magnetic field data, we are interested in identifying the least noisy sensor for every time step (see Section 2).
%, such that the best reward is for the arm with the least noise. 

The dataset was generated with contexts \(x_t\) consisting of measured variables associated with the electrical behavior of the GRIFEX spacecraft \citep{norton2012spaceborne, cutler2015grfx}, and reward is the negative of the magnitude of the sensor noise measured at every time step.

% The dataset was generated in a lab setup using telemetry (contexts) downloaded from the GRIFEX satellite \citep{norton2012spaceborne, cutler2015grfx}.  
%A physical experimental setup consisted of magnetic coils driven using non-ideal models of reward that would be seen in a spacecraft. 
Data were collected using 3 sensors (arms), and sensor readings were downloaded for all three sensors at all times steps, although the algorithm does not know these in advance and must select one sensor at each time step. The context information was used in conjunction with a realistic simulator to generate spacecraft magnetic field, and hence a realistic model of sensor noise, as a function of context. The true magnetic field was computed using models of the earth's magnetic field.

\begin{figure}[htbp]
\centering
        \includegraphics[width=0.44\textwidth]{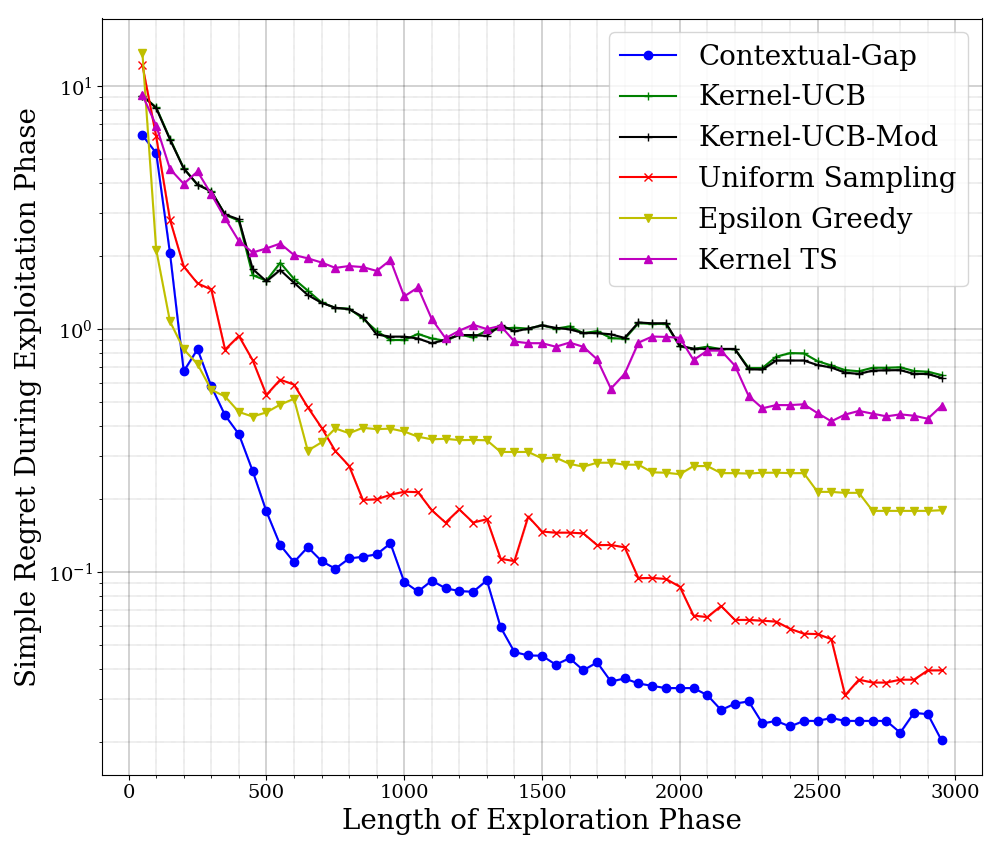}
        \caption{Average Simple Regret Evaluation on Spacecraft Magnetic Field Dataset}
        \label{fig:Spacecraft_alpha_1}
\end{figure}

Histogram of number times the best, second best and third best (worst) arms are selected during exploration is shown in Figure \ref{fig:histogram_spacecraft_labexperimental}. As expected, algorithms designed to minimize cumulative regret focus on the best arm more and Contextual-Gap explores best and second best arms.
Figure \ref{fig:Spacecraft_alpha_1} shows the simple regret minimization curves for the spacecraft data-set and even in this case Contextual-Gap converges faster compared to other algorithms. 
%\blue{Note that, in addition to the non-i.i.d nature, there exists large variability in reward for certain regions of the context space}.

%% file: 7_conclusion.tex
%-------------------------------------------------------------------------------------%
\section{Conclusion}
\label{sec:conclusion_future_work}
%-------------------------------------------------------------------------------------%
In this work, we present a novel problem: that of simple regret minimization in the contextual bandit setting. We propose the Contextual-Gap algorithm, give a regret bound for the simple regret, and show empirical results on three multiclass datasets and one lab-based spacecraft magnetometer dataset.  It can be seen that in this scenario persistent and efficient exploration of the best and second best arms with the Contextual-Gap algorithm provides improved results compared against algorithms designed to optimize cumulative regret.

%% file: 8_appendix.tex
\section{Remarks}
\begin{itemize}
    \item In Section \ref{sec:filtration}, we formalize the notion of History \( H_{t}\) in the main paper.
    %\item In supplementary material, we define  \( B_{a,t}(x_t) =   \max_{i \neq a} U_{i,t}(x_t) - L_{a,t}(x_t) \) for the convenience. In this case \( J_t(x) \) is the arm that minimizes \( B_{a,t}(x_t) \). 
    \item Detailed comments about constant \( C_2 \) (Theorem 4.1 of the main paper) are at the end of Subsection \ref{subsec:MainTheoremProof}.
    \item Detailed experiments and results about different \( \alpha \) are in Section \ref{sec:DetailedExperiments}.
\end{itemize}

\section{Probabilistic Setting and Martingale Lemma}
\label{sec:filtration}
For the theoretical results, the following general probabilistic framework is adopted, following \citet{abbasi2011improved} and \citet{durand2017streaming}. We formalize the notion of history \( H_{t}\) defined in the Section 3 of the main paper using filtration. A filtration is a sequence of \(\sigma\)-algebras \(\{\mathcal{F}_t\}_{t=1}^{\infty}\) such that \(\mathcal{F}_1 \subseteq \mathcal{F}_2 \subseteq \cdots \subseteq \mathcal{F}_n \subseteq \cdots \). Let \(\{\mathcal{F}_t\}_{t=1}^{\infty}\) be a filtration such that $x_t$ is $\mathcal{F}_{t-1}$ measurable, and $\zeta_t$ is $\mathcal{F}_{t}$ measurable. For example, one may take \(\mathcal{F}_t:= \sigma(x_1, x_2, \cdots, x_{t+1}, \zeta_1, \zeta_2, \cdots, \zeta_t)\), i.e., \( \mathcal{F}_t \) is the \( \sigma-\)algebra generated by  \( x_1, x_2, \cdots, x_{t+1}, \zeta_1, \zeta_2, \cdots, \zeta_t\).  

We assume that \(\zeta_t\) is a zero mean, \(\rho\)-conditionally sub-Gaussian random variable, i.e., \(\zeta_t\) is such that for some \(\rho > 0\) and \(\forall \gamma \in \RR \),
\begin{equation}
\EE[e^{\gamma\zeta_t}|\mathcal{F}_{t-1}] \leq \exp\bigg(\frac{\gamma^2 \rho^2}{2} \bigg).
\end{equation}

\begin{definition}[Definition 4.11 in \citet{motwani_raghavan_1995}]
Let \( (\Sigma,\mathcal{F}, Pr) \) be a probability space with filtration \( \mathcal{F}_0, \mathcal{F}_1, \dots \). Suppose that \( Z_0, Z_1, \dots \) are random variables such that for all \( i > 0\), \( Z_i \) is \( \mathcal{F}_i \) measurable. The sequence  \( Z_0, Z_1, \dots \) is a martingale provided for all \( i \geq 0 \),
\[
\EE[Z_{i+1}|\mathcal{F}_i] = Z_i.
\]
\end{definition}

%\blue{citep this in the main paper, not needed here}.
 \begin{lemma}[Theorem 4.12 in \citet{motwani_raghavan_1995}]
 Any subsequence of a martingale is also a martingale (relative to the corresponding subsequence of the underlying filter).
 \end{lemma}
The above Lemma is important because we construct confidence intervals for each arm separately. Note that we define a subset of time indices ( \(D_{a,t}\) of each arm \(a\)), when the arm \( a \) was selected. Based on these indices we can form sub-sequences of the main context \(\{x_t\}_{t=1}^\infty\) and noise sequence \(\{\zeta_t\}_{t=1}^\infty\) such that the assumptions on the main sequence hold for subsequences.

\subsection{Theorem 4.1 in Main Paper}
Theorem 4.1 is a slight modification of Theorem 2.1 in \citet{durand2017streaming}. In the contextual bandit setting in \citet{durand2017streaming}, for any \(\delta \in (0,1] \), Theorem 2.1 in \citet{durand2017streaming} establishes that with probability at least \(1-\delta\), it holds simultaneously over all \(x \in \sX\) and \(t \geq 0\),
\begin{multline*}
|f_a(x) - \hat{f}_{a,t}(x)| \leq \\ \frac{\hat{\sigma}_{a,t}(x)}{\sqrt{\lambda}}\Bigg[ \sqrt{\lambda}\n{f_a}_{\sH} + \rho\sqrt{2\ln(1/\delta) + 2 \gamma_t(\lambda)} \Bigg],
\end{multline*}
where \( \gamma_t(\lambda) = \frac{1}{2} \sum_{\tau=1}^t \ln( 1+ \frac{1}{\lambda} \hat{\sigma}_{a,\tau-1}(x_\tau) )\)

For  \( T  \geq t \), one can replace \(t\) in the log terms with \(T\). Then \( \forall x, \forall t \geq 1 \), we have

\begin{multline*}
 1 - \delta \leq \PP \Bigg( |f_a(x) - \hat{f}_{a,t}(x)| \leq \\ \frac{\hat{\sigma}_{a,t}(x)}{\sqrt{\lambda}}\Big[ \sqrt{\lambda}\n{f_a}_{\sH} + \rho\sqrt{2\ln(1/\delta) + 2 \gamma_T(\lambda)} \Big]  \Bigg). 
\end{multline*}
Let \(\delta = e^{-\beta^2}\). In that case, 
\begin{multline*}
1 - e^{-\beta^2} \leq \PP \Bigg( |f_a(x) - \hat{f}_{a,t}(x)| \leq \\ \frac{\hat{\sigma}_{a,t}(x)}{\sqrt{\lambda}} \Big[ \sqrt{\lambda}\n{f_a}_{\sH} + \rho\sqrt{2 \beta^2 + 2 \gamma_T(\lambda)} \Big]  \Bigg).
\end{multline*}
Using triangle inequality \( \sqrt{p+ q} \leq \sqrt{p} + \sqrt{q} \) for any \( p, q \geq 0 \),
\begin{multline*}
 1 - e^{-\beta^2} \leq  \PP \Bigg( |f_a(x) - \hat{f}_{a,t}(x)| \leq \\ \frac{\hat{\sigma}_{a,t}(x)}{\sqrt{\lambda}} \Big[ \sqrt{\lambda}\n{f_a}_{\sH} + \rho\sqrt{2 \beta^2} + \rho\sqrt{ 2 \gamma_T(\lambda)} \Big]  \Bigg).
\end{multline*}
 Let \(C_1 = \rho\sqrt{2}  \) and \( C_2 = \sqrt{\lambda}\n{f_a}_{\sH}  + \rho\sqrt{ 2 \gamma_T(\lambda)}\). Hence, we have
 \begin{multline*}
 1 - e^{-\beta^2} \leq  \PP \Bigg( |f_a(x) - \hat{f}_{a,t}(x)| \leq \\ \frac{\hat{\sigma}_{a,t}(x)}{\sqrt{\lambda}} [ C_1\beta + C_2 ] \Bigg). 
\end{multline*}

\section{Lower Bound on \texorpdfstring{\(r^{\textnormal{th}}\)}{rth} Eigenvalue}

First we state the Lemmas that we use to prove Lemma 5.1 in main paper. 
\begin{lemma}[Lemma 9 in \citet{li2017online}] \label{lem:tGeqLogt}
		If \( a>0, b>0, ab\geq e \), then for all \( t \geq 2a\log(ab) \),
		\begin{equation}
			\label{eq:tGeqLogt}
			t \geq a\log(bt).
		\end{equation}
	\end{lemma}
 
 \begin{lemma}[Lemma 1.1 in \citet{zi2009schur}]\label{lem:zi2009schur}
Let $A \in \mathbb{R}^{n \times n}$ be a symmetric positive definite matrix partitioned according to	\[
	A =
	\left[
	\begin{array}{c|c}
	A_{11} & A_{12} \\
	\hline
	A_{12}^T & A_{22}
	\end{array}
	\right]
	,\] where $A_{11} \in  \mathbb{R}^{(n-1) \times (n-1)}, A_{12} \in \mathbb{R}^{(n-1)}$ and $A_{22} \in  \mathbb{R}^1$. Then $\det(A) = \det(A_{11}) (A_{22} - A_{12}^TA_{11}^{-1}A_{12}) $.
\end{lemma}

%extension of Horn's inequality to compact self-adjoint operators, 
\begin{lemma}[Special case of extended Horn's inequality (Theorem 4.5 of \citet{bercovici2009horn})] \label{lem:Weyl_Infinite}
 Let \( A,B \) be compact self-adjoint operators. Then for any \( p \geq 1 \),
 \begin{equation}
 \lambda_p(A+B) \leq \lambda_{1}(A) + \lambda_p(B).
 \end{equation}
 \end{lemma}

\begin{theorem}[Freedman's inequality for self adjoint operators, Thm 3.2 \& section 3.2 in \citet{minsker2017some}]\label{thm:SelfAdjointOpFreedman} Let \( \{ \Phi_t\}_{t = 1,...} \) be a sequence of self-adjoint Hilbert Schmidt operators \( \Phi_t : \sH \rightarrow \sH \) acting on a seperable Hilbert space ( \( \EE \Phi \) is a operator such that \( \langle (\EE \Phi)z_1, z_2 \rangle_{\sH}  = \EE \langle \Phi z_1, z_2 \rangle_{\sH}  \) for any \( z_1,z_2 \in \sH \)). Additionally, assume that \( \{ \Phi_t\}_{t = 1,...} \) is a martingale difference sequence of %with values in \( \dim(\sH) \times \dim(\sH) \) 
self adjoint operators such that \( \n{\Phi_t} \leq L^2 \) almost surely for all \( 1 \leq t \leq T\) and some positive \( L \in \RR \). Denote by \( W_t  = \sum_{s = 1}^t \EE_{s-1}[\Phi_s^2 ] \) and \( p(t) = \min(-t,1) \)  . Then for any \( a \geq \frac{1}{6} (L^2 + \sqrt{L^4 + 36 b}), b \geq 0 \), 
\begin{multline*} 
\PP\left(  \n{ \sum_{j = 1}^t \Phi_j } > a \text{ and } \lambda_{1}(W_t) \leq b \right) \leq \\ \tilde{d}\cdot \exp\left( - \frac{a^2/2}{b + aL^2/3} \right),
\end{multline*}
where \( \n{\cdot} \) is the operator norm and \( \tilde{d}:= 50 \sum_{r=1}^{\infty} (p(-\frac{a\lambda_r(\EE W_t)}{L^2 b}))\). 
\end{theorem}

Note that \( \tilde{d} \) is a function of \( t \) but it's upper bounded by \( d^* \) which is the rank of \(\EE_{s-1}[\phi(x)\phi(x)^T]\).
%\blue{We will have to give some intuition about the quantity \( \tilde{d} = 50 \trace(p(-\frac{a \EE W_t}{L^2 \sigma^2})) \) }
%\blue{Can you add whatever you have written in the main paper? Also citep the Minsker's paper and say for more details please go there.}

%We know that \( \n{A} \leq \n{A}_{HS} =  \trace(A^*A)\) and \( \lambda_{max}(A) \leq  \n{A}\). 

\subsection{Proof of Lemma 5.1 in main paper}
Lemma 7 in \citet{li2017online} gives the lower bound on minimum eigenvalue (finite dimensional case) when reward depends linearly on context. We extend it to \( r^{th} \) largest eigenvalue (infinite dimensional case) and the case when reward depends non-linearly on context. 

 \begin{proof}
 %\blue{ I'm not even sure if following definitions make sense. What it means to take expectation of an operator and sum of operators. We will need to define this carefully.}
 
 \(\sX \subset \Rbb^d\) is a compact space endowed with a finite positive Borel measure. For a continuous kernel \(k\) the canonical feature map \(\phi\) is a continuous function \(\phi:\sX \rightarrow \sH\), where \(\sH\) is a separable Hilbert space (See section 2 of \citet{micchelli2006universal} for a construction such that \(\sH\) is separable). In such a setting \(\phi(\sX)\) is also compact space with a finite positive Borel measure \citep{micchelli2006universal}. We now define a few terms on \(\phi(\sX)\).
 %\begin{itemize}
 %\item \( \mathbb{E}_t[\cdot] := \mathbb{E}[\cdot | \phi(x_1), ... , \phi(x_t)] \)
 %\item There exists a sequence \( \{\lambda_{x}\}_{r=1}^\infty \) such that \( \sum_{s=1}^t \mathbb{E}_{t-1}[\phi(x_t)\phi(x_t)^T] \geq t \lambda_{x}, \forall t \) and an integer \(r^\ast \in [1,\infty)\) such that \(\lambda_{x} > 0 \) for \(r \leq r^\ast\). \blue{We have to put this in the Theorem statement}  
 
 Define the random variable \( \Phi_t := \mathbb{E}_{t-1}[\phi(x_t)\phi(x_t)^T] - \phi(x_t)\phi(x_t)^T \). Let \( Z_t := \sum_{s=1}^{t} \Phi_s   = \sum_{s=1}^{t} \mathbb{E}_{s-1}[\phi(x_t)\phi(x_t)^T] - S_t = V_t - S_t\).  
  
 By construction, \( \{ Z_t\}_{t = 1,2,...} \) is a martingale and \( \{\Phi_s\}_{s= 1,2,...} \) is the martingale difference sequence. Notice that \( \lambda_{1}(\Phi_t) \leq L^2\). To use the Freedman's inequality, we  lower bound the operator norm of \(Z_t\), \(\n{Z_t} \) and upper bound the largest eigenvalue of \(W_t\), \( \lambda_{1}(W_t)\). Let \( \nu(A) = \max_{i} | \lambda_i(A) |\) be the spectral radius of operator \( A\). We work with the spectral radius because it is not necessary that \(Z_t\) is a positive definite operator. It is well known that
 \begin{equation}\label{eq:spectral_radius}
\nu(A) \leq  \n{A} .
 \end{equation} 
 
By assumption \ref{Alambda}, \(\EE_{s-1}[\phi(x)\phi(x)^T]\) lies in a fixed \(d^\ast\) dimensional subspace with its eigenvalues \(\lambda_r(\EE_{s-1}[\phi(x)\phi(x)^T]) > \lambda_x\) for \(r \leq d^\ast\). Thus, for \(V_t = \sum_{s=1}^{t}\EE_{s-1}[\phi(x)\phi(x)^T]\), \(\lambda_r(V_t) \geq t \lambda_x\).

\textbf{Bound on \( \n{Z_t} \) : }
By definition, \( V_t  =  Z_t + S_t \). Hence, \( \lambda_r(V_t)   \leq   \lambda_{1}(Z_t)  + \lambda_{r}(S_t) \) by using Horn's inequality (Lemma \ref{lem:Weyl_Infinite}). 
        \begin{eqnarray}  \nonumber
       \lambda_{1}(Z_t)  &\geq&  \lambda_r(V_t)  - \lambda_{r}(S_t) \\  \nonumber
     \lambda_{1}(Z_t)      & \geq & t \lambda_{x} - \lambda_{r}(S_t) \\ \nonumber 
        \nu(Z_t)      & \geq & t \lambda_{x} - \lambda_{r}(S_t),
    \end{eqnarray}
    where the second step is due to \ref{Alambda} and the third step is by definition of spectral radius.  
    By Eqn. \eqref{eq:spectral_radius}, we have 
    \begin{eqnarray} \label{eq:lambdaofZt}
       \n{Z_t}     & \geq & t \lambda_{x} - \lambda_{r}(S_t).
		\end{eqnarray}

 \textbf{Bound on \( \lambda_{1}(W_t) \) : }
To bound the term \( \lambda_{1}(W_t) \), write 
 \begin{eqnarray} \nonumber
			W_t &= & \sum_{s=1}^{t} \EE_{s-1}[\Phi_s^2] \\ \nonumber
             &= & \sum_{s=1}^{t} \EE_{s-1}\big[ (\EE_{s-1} [\phi(x_s)\phi(x_s)^T] - \phi(x_s)\phi(x_s)^T)^2\big].
           \end{eqnarray}  
           By using square expansion, 
      \begin{eqnarray}        \nonumber
         W_t   &= & \sum_{s=1}^{t} \EE_{s-1}\big[ (\EE_{s-1} [\phi(x_s)\phi(x_s)^T]^2 + (\phi(x_s)\phi(x_s)^T)^2  \\ \nonumber
          & & -  \EE_{s-1}[\phi(x_s)\phi(x_s)^T](\phi(x_s)\phi(x_s)^T) \\ \nonumber
          & & - (\phi(x_s)\phi(x_s)^T) \EE_{s-1}[\phi(x_s)\phi(x_s)^T] \big]\\ \nonumber
            & = & \sum_{s=1}^{t}  \EE_{s-1}[(\phi(x_s) \phi(x_s)^{T})^2] -  \EE_{s-1}[\phi(x_s) \phi(x_s)^{T}]^2. 
		\end{eqnarray}
        Taking norm on both sides,
          \begin{equation}   \nonumber  
         \n{W_t}   =  \n{\sum_{s=1}^{t}  \EE_{s-1}[(\phi(x_s) \phi(x_s)^{T})^2] -  \EE_{s-1}[\phi(x_s) \phi(x_s)^{T}]^2} .
         \end{equation}
         As both terms on the right hand side are positive semi-definite matrices, 
         \begin{equation} \nonumber
        \n{W_t} \leq  \n{\sum_{s=1}^{t}  \EE_{s-1}[(\phi(x_s) \phi(x_s)^{T})^2]} .
		\end{equation}
   
        Next, we use convexity properties of norms to get the upper bound.    
		\begin{equation*}
        \begin{aligned}
        \n{W_t}  & \leq \sum_{s=1}^{t} \n{\EE_{s-1}[(\phi(x_s) \phi(x_s)^{T})^2]} \\
        & = \sum_{s=1}^t \n{\EE_{s-1}[(\phi(x_s) \big(\phi(x_s)^{T}\phi(x_s)\big) \phi(x_s)^{T})]}\\ 
   		& \leq L^2 \sum_{s=1}^{t} \n{\EE_{s-1}[(\phi(x_s) \phi(x_s)^{T})]} \\
       & \leq L^2 \sum_{s=1}^{t} \EE_{s-1}[\n{(\phi(x_s) \phi(x_s)^{T})}]
      \end{aligned}
        \end{equation*}
        where the first step is due to the triangle inequality and the third step is due to the upper bound \(\n{\phi(x)} \leq L\), the fourth step is due to the convexity of the operator norm and Jensen's inequality. Using the properties of Hilbert Schmidt operators, we can write
\begin{eqnarray*}
		\EE_{s-1}[\n{(\phi(x_s) \phi(x_s)^{T})}] &\leq& \EE_{s-1}[\n{(\phi(x_s) \phi(x_s)^{T})}_{HS}] \\
		&=& \EE_{s-1}[\n{\phi(x_s)}^2] \leq L^2
\end{eqnarray*}
Therefore, we can bound the norm \(\n{W_t}\) as
        \begin{equation*}
        \begin{aligned}
        \n{W_t} & \leq L^2 \sum_{s=1}^{t} L^2 \\
      		& = tL^4, 
        \end{aligned}
        \end{equation*}

        Again, by using Eqn. \eqref{eq:spectral_radius}, we have
       \begin{eqnarray} \label{eq:lambdaofWt}
     \lambda_{1}(W_t) & \leq & tL^4. 
        \end{eqnarray}

Now, we shall construct a parameter \(A\) such that
    \begin{equation}\label{eq:Amdelta}
    \frac{a^2/2}{b + aL^2/3} \geq A.
    \end{equation}
For this inequality to hold, one can see, by its quadratic solution, \( a \geq f(A,b):= \frac{1}{3} AL^2 + \sqrt{\frac{1}{9}A^2L^4 +2Ab } \). Note that for \(A > 1\), the condition of \(a \geq f(A,b)\) also satisfies the conditions of Friedman's inequality in Theorem \ref{thm:SelfAdjointOpFreedman}.
    
    Let $A(m,\delta) = \log\frac{(m+1)(m+3)}{\delta}$ and \( P \) be the probability of event \(  \Big[ \exists t: \lambda_{r}(\bS_t) \leq t\lambda_{x}- f(A(tL^4, \delta), tL^4)  \Big]\). 
		\begin{eqnarray}  \nonumber
			&P& \\ \label{eq:SelfAdjointInterIneq0}
			&=& \PP \Big[ \exists t: \lambda_{r}(\bS_t) \leq t\lambda_{x}- f(A(tL^4, \delta), tL^4) \Big] \\ \nonumber
			&\le & \PP \Big[\exists t: \lambda_{r}(\bS_t) \leq t\lambda_{x} \\  \label{eq:SelfAdjointInterIneq1} 
			& & - f(A(\lambda_{1}(W_t), \delta), \lambda_{1}(W_t)) \Big] \\  \nonumber
			&\le &\sum_{m=0}^{\infty} \PP \Big[\exists t: \lambda_{r}(\bS_t) \leq t\lambda_{x}  \\  \label{eq:SelfAdjointInterIneq2}
			& & - f(A(m, \delta), m), \lambda_{1}(W_t) \leq m  \Big] \\   \nonumber
			&\le &\sum_{m=0}^{\infty} \PP \Big[ \exists t: \n{Z_t} \geq f(A(m, \delta), m), \\ \label{eq:SelfAdjointInterIneq3}
			& & \lambda_{1}(W_t) \le m  \Big] \\   \label{eq:SelfAdjointInterIneq4}
			&\le & \tilde{d}\sum_{m=0}^{\infty} \exp\left(- A(m, \delta) \right) \\ \nonumber
			&= &  \tilde{d}\sum_{m=0}^{\infty} \frac{\delta}{(m+1)(m+3)} \\
            &\leq& \tilde{d} \cdot \delta,
		\end{eqnarray}
		where \eqref{eq:SelfAdjointInterIneq1} is because \( A \) is increasing in \( m \), \( f \) is increasing in \( A, b \), and Eqn. \eqref{eq:lambdaofWt}. Eqn. \eqref{eq:SelfAdjointInterIneq2} is by application of the union bound over all the events for which \(\lambda_{1}(W_t) \le m \). Also, Eqn. \eqref{eq:SelfAdjointInterIneq3} is due to Eqn. \eqref{eq:lambdaofZt} and Eqn. \eqref{eq:SelfAdjointInterIneq4} is due to Theorem \ref{thm:SelfAdjointOpFreedman}.

		The result is obtained by replacing \( \delta \) by \( \frac{\delta}{\tilde{d}} \).
        
    For the second part. Let \(\tilde{\lambda}_x:= \frac{\lambda_x}{L} \). By definition of \(L\), \(\tilde{\lambda}_x \leq 1 \). Let \( t \geq  \frac{256}{\Tlambda_{x}^2}\log\frac{128\tilde{d}}{\Tlambda_{x}^2 \delta}  \). Then by using the Lemma \ref{lem:tGeqLogt}, 
    \begin{equation}\label{eq:tlambdaCondition}
    t \geq \frac{128}{\Tlambda_{x}^2} \log\frac{t\tilde{d}}{\delta}.
    \end{equation}
 	Rearranging the terms, we get 
    \begin{equation*}
        \frac{t\Tlambda_{x}^2}{4}    \geq  32 \log\frac{t\tilde{d}}{\delta} 
        \end{equation*}
     Taking square root and then multiplying by \( \sqrt{t} \) on both sides
      \begin{eqnarray*}
        \frac{t\Tlambda_{x}}{2}    &\geq&  \sqrt{32t \log\frac{t\tilde{d}}{\delta} } \\ \nonumber
            &=& \frac{2}{3} \sqrt{72t \log\frac{t\tilde{d}}{\delta} } \\ \nonumber
           &=& \frac{2}{3} \sqrt{36t\log\frac{t\tilde{d}}{\delta} + 36t \log\frac{t\tilde{d}}{\delta}  }.
      \end{eqnarray*}
Using equation \eqref{eq:tlambdaCondition},   
      \begin{equation*}
      \begin{aligned}
           \frac{t\Tlambda_{x}}{2} &\geq \frac{2}{3} \sqrt{36t\log\frac{t\tilde{d}}{\delta} + \frac{36 \cdot 128}{\Tlambda_x^2} \Big(\log\frac{t\tilde{d}}{\delta} \Big)^2  } \\ \nonumber
           &= \frac{2}{3} \sqrt{36t\log\frac{t\tilde{d}}{\delta} + \frac{36 \cdot 32}{\Tlambda_x^2} 4 \Big(\log\frac{t\tilde{d}}{\delta}\Big)^2  }.
        \end{aligned}
        \end{equation*}
        Since \( \Tlambda_x^2 \leq 1\) we have 
          \begin{eqnarray} \nonumber
        \frac{t\Tlambda_{x}}{2} &\geq& \frac{2}{3} \sqrt{36t\log\frac{t\tilde{d}}{\delta} + (36 \cdot 32) \Big(2\log\frac{t\tilde{d}}{\delta}\Big)^2  } \\ \label{eq:tlambda_x}
        &>& \frac{2}{3} \sqrt{18t \cdot 2 \log\frac{t\tilde{d}}{\delta} +  \Big(2 \log\frac{t\tilde{d}}{\delta}\Big)^2  }.
        \end{eqnarray}
        
      Now we use the condition on \( \delta \) as stated in the Theorem statement: \( 0 \leq \delta \leq \frac{1}{8}\). We can see that 
      \begin{equation}\label{eq:functiont}
      \frac{1}{8} \leq \frac{t^2 \tilde{d}^2}{(t+1)(t+3)},
      \end{equation}  because \(\frac{t^2 \tilde{d}^2}{(t+1)(t+3)} \) is a monotonically increasing function for both \( t,\tilde{d} \) for \( t,\tilde{d} \geq 1\). Simplifying Eqn. \eqref{eq:functiont}, we get
       \begin{equation}\label{eq:functiont1}
      \frac{t^2\tilde{d}^2}{\delta^2} \geq \frac{(t+1)(t+3)}{\delta} .
      \end{equation} 
      Taking log of both sides,
      \begin{equation} \nonumber
       2\log\frac{t\tilde{d}}{\delta} \geq \log(\frac{(t+1)(t+3)}{\delta}) = A(t,\delta).
      \end{equation} 
      
    Without loss of generality, we will assume that \(L = 1\). From Eqn. \eqref{eq:functiont1}  and Eqn. \eqref{eq:tlambda_x}, we have
                \begin{eqnarray} \nonumber
        \frac{t\lambda_{x}}{2} &\geq& \frac{2}{3} \sqrt{18t \cdot A(t,\delta) +   A(t,\delta)^2  } \\ \nonumber
        &=& \frac{1}{3} \sqrt{18t \cdot A(t,\delta) +   A(t,\delta)^2  } \\ \nonumber
        & & + \frac{1}{3} \sqrt{18t \cdot A(t,\delta) +   A(t,\delta)^2  } \\ \nonumber
        &\geq & \frac{1}{3} \sqrt{18t \cdot A(t,\delta) +   A(t,\delta)^2  } + \frac{1}{3}A(t,\delta)
        \end{eqnarray}
   Therefore,
   \begin{equation}
   \label{eq:tlambda_xfinal}
       \frac{t\lambda_{x}}{2} \geq f(A(t,\delta),t).
   \end{equation}
     Equations \eqref{eq:SelfAdjointInterIneq0} and \eqref{eq:tlambda_xfinal} complete the proof. 
 \end{proof}

 \section{Monotonic Upper bound of \texorpdfstring{\( s_{a,t}(x) \)}{s(x)} }

\begin{lemma}\label{lem:AMGM}[ Arithmetic Mean-Geometric Mean Inequality \citep{steele2004cauchy}] 
For every sequence of nonnegative real numbers $ a_1,a_2,...a_n$ one has 
\[ (\prod_{i=1}^n a_i)^{1/n} \leq \frac{\sum_{i=1} a_i}{n} \]
with equality if and only if $ a_1 = a_2 = ...=  a_n$.  
\end{lemma}

\begin{lemma} \label{lem:RatioDeterminantBound1}
If \(\lambda_1 \geq \lambda_2 \geq \cdots \geq \lambda_d >0 \), and \(\mu_1 \geq 0, \mu_2 \geq 0 \cdots \mu_d \geq 0 \) such that \( \sum_j \mu_j = L \) and \( \lambda_d \geq L\) then
\begin{equation*}
\prod_{i=1}^d \bigg( 1 + \frac{\mu_i}{\lambda_i} \bigg)  -1 \leq \frac{2L}{\lambda_d}.
\end{equation*}
\end{lemma}
 
\begin{proof}
By replacing each \(\lambda_i \) with the smallest element \( \lambda_d\) we get, 
\begin{equation*}
\begin{aligned}
\prod_{i=1}^d \bigg( 1 + \frac{\mu_i}{\lambda_i} \bigg) -1 
&\leq  \prod_{i=1}^d \bigg( 1 + \frac{\mu_i}{\lambda_d} \bigg) -1  \\
&=  \prod_{i=1}^d \bigg( \frac{\lambda_d + \mu_i}{\lambda_d} \bigg) -1 \\
&=   \bigg( \frac{ \prod_{i=1}^d (\lambda_d + \mu_i)}{\lambda_d^d} \bigg) -1 \\ 
	&\leq \bigg( \frac{ \sum_{i=1}^d  (\lambda_d + \mu_i)}{d \lambda_d} \bigg)^d  -1  \\
 	&= \bigg( \frac{ d \lambda_d + L }{d \lambda_d} \bigg)^d -1 \\
    &= \bigg( 1  + \frac{ L }{d \lambda_d} \bigg)^d -1 \\
    &\leq e^{L/\lambda_d} -1 ,
\end{aligned}
\end{equation*}
where the fourth inequality is by Lemma \ref{lem:AMGM} and last inequality holds because \( (1+ \frac{a}{x})^x \) approaches \( e^a\) as \(x \rightarrow \infty\) and \( (1+ \frac{a}{x})^x \) is a monotonically increasing function of \( x\).

By  \( e^x \leq 1+ 2x \) for \( x \in [0,1] \)  and the assumption that \(  \lambda_d \geq L \),
\begin{equation*}
\begin{aligned}
\prod_{i=1}^d \bigg( 1 + \frac{\mu_i}{\lambda_i} \bigg) -1 
&\leq e^{L/\lambda_d} -1  \\
&\leq  1 + \frac{2L}{\lambda_d} - 1 \\
&= \frac{2L}{\lambda_d}.
\end{aligned}
\end{equation*}

\end{proof}

\subsection{Proof of Lemma 5.2 in main paper}
\begin{proof}
%Note: To reduce levels of subscripts, we write \(I_{a,t}:= I_{N_{a,t}}\).
We will assume that \(L=1\). 
We write \[
	K_{a,t+1} + \lambda_{I_{a,t+1}} =
	\left[
	\begin{array}{c|c}
	K_{a,t}+ \lambda_{I_{a,t}} & k_{a,t}(x) \\
	\hline
	k_{a,t}(x)^T & k(x,x) + \lambda
	\end{array}
	\right]
	.\]
    
Let  \( \mu_i = \lambda_i(K_{a,t+1} + \lambda_{I_{a,t+1}}) - \lambda_i(K_{a,t}+ \lambda_{I_{a,t}} ).\)
%\( \mu_i \) be the difference between \( i^{th}\) biggest eigenvalues of \( K_{a,t+1} + \lambda_{I_{a,t+1}} \) and \( K_{a,t}+ \lambda_{I_{a,t}}  \).  

Using Lemma \ref{lem:zi2009schur}, 
\begin{eqnarray} \nonumber
 & & \det(K_{a,t+1} + \lambda_{I_{a,t+1}}) \\ \nonumber
 &=&  \det(K_{a,t}+ \lambda_{I_{a,t}}) \Big(k(x,x) + \lambda \\ \nonumber
& & - k_{a,t}(x)^T(K_{a,t}+ \lambda_{I_{a,t}})^{-1}k_{a,t}(x) \Big).
\end{eqnarray}
Rearranging, 
\begin{eqnarray} \nonumber
 & & k(x,x) - k_{a,t}(x)^T(K_{a,t}+ \lambda_{I_{a,t}})^{-1}k_{a,t}(x) \\ \nonumber
 &=& \frac{\det(K_{a,t+1} + \lambda_{I_{a,t+1}})}{\det(K_{a,t}+ \lambda_{I_{a,t}}) } - \lambda. 
\end{eqnarray}
Dividing both sides by \( \lambda\), 
\begin{eqnarray} \nonumber
& & \frac{k(x,x) - k_{a,t}(x)^T(K_{a,t}+ \lambda_{I_{a,t}})^{-1}k_{a,t}(x)}{\lambda} \\  \label{eq:sigma_hat} 
&=&  \frac{\det(K_{a,t+1} + \lambda_{I_{a,t+1}})}{\lambda \det(K_{a,t}+ \lambda_{I_{a,t}}) } - 1 .
\end{eqnarray}
Notice that the left hand side is equal to \( \frac{\hat{\sigma}_{a,t}(x)}{\lambda} \).
Using the definitions of \( s_{a,t}(x) \) and \( \hat{\sigma}_{a,t}(x) \), we can write, 
\begin{eqnarray*}
s_{a,t}(x)^2 &= &  4 (C_1 \beta + C_2)^2 \frac{\hat{\sigma}_{a,t}(x)^2}{\lambda}   \\
&=& 4 (C_1 \beta + C_2)^2  \Big( \frac{\det(K_{a,t+1} + \lambda I_{a,t+1})}{\lambda \det(K_{a,t} + \lambda I_{a,t})} - 1 \Big) \\
&=& 4  (C_1 \beta + C_2)^2 \Big(  \frac{\prod_{i = 1}^{N_{a,t}+1} \lambda_{i,a,t+1} }{\lambda \prod_{i = 1}^{N_{a,t}} \lambda_{i,a,t} } - 1 \Big)
\end{eqnarray*}
By assumption in the statement of the Lemma, \(N_{a,t}  \geq d^\ast \). Hence, all eigenvalues above \( d^\ast \) are \( \lambda \).

By replacing all eigenvalues \( \lambda_{i,a,\tau}\) by \( \lambda\) for \( \tau = \{t,t+1 \}\) and \( i > d^\ast\), we get
\begin{eqnarray*}
s_{a,t}(x)^2 &=& 4 (C_1 \beta + C_2)^2 \Big( \prod_{i =1}^{d^\ast}  \frac{\lambda_{i,a,t+1}}{ \lambda_{i,a,t}}- 1 \Big).
\end{eqnarray*}
Note that \(\lambda_{i,a,t+1} = \lambda_{i,a,t} + \mu_i\). By replacing \(\lambda_{i,a,t+1} \), we get
\begin{eqnarray*}
s_{a,t}(x)^2  &=& 4 (C_1 \beta + C_2)^2 \Big( \prod_{i =1}^{d^\ast}  \frac{\lambda_{i,a,t} + \mu_i}{ \lambda_{i,a,t}}- 1 \Big)\\
 &= & 4 (C_1 \beta + C_2)^2 \Bigg(  \prod_{i = 1}^{d^\ast} \Big( 1+ \frac{ \mu_i }{ \lambda_{i,a,t} } \Big) - 1 \Bigg)\\
&\leq &  4 (C_1 \beta + C_2)^2 \Bigg(  1 + \frac{2L}{\lambda_{d^\ast,a,t}} - 1 \Bigg),
\end{eqnarray*}
where the third inequality is due to Lemma \ref{lem:RatioDeterminantBound1}.  

For \(L = 1\), 
\begin{eqnarray*}
s_{a,t}(x)^2 &\leq &  4 (C_1 \beta + C_2)^2 \Bigg(  1 + \frac{2}{\lambda_{d^\ast,a,t}} - 1 \Bigg) \\
& = &  4 (C_1 \beta + C_2)^2 \Bigg(  \frac{2}{\lambda_{d^\ast,a,t}} \Bigg). 
\end{eqnarray*}

Note that \( \lambda_{d^\ast,a,t} = \lambda_{d^\ast}(K_{a,t+1} + \lambda_{I_{a,t+1}}) =  \lambda_{d^\ast}(K_{a,t+1}) + \lambda\). By Lemma 5.1 in main paper \(\lambda_{d^\ast}(K_{a,t+1}) \geq  N_{a,t}\lambda_{x}\). We can apply Lemma  \ref{lem:RatioDeterminantBound1} only when
\begin{eqnarray*}
\frac{1}{\lambda+N_{a,t}\lambda_{x}/2} &<& 1 
\end{eqnarray*}
or 
\begin{eqnarray*}
N_{a,t} & > &  \frac{2(1 -\lambda) }{\lambda_{x}}. 
\end{eqnarray*}

The assumption in the statement of the lemma satisfies the above equation. Hence, we have
\begin{eqnarray*}
s_{a,t}(x)^2 &\leq & 4 (C_1 \beta + C_2)^2 \Bigg(  \frac{2}{\lambda+ N_{a,t} \lambda_{x}/2} \Bigg) \\
&= & 8 (C_1 \beta + C_2)^2 \Bigg(   \frac{ 1}{\lambda+N_{a,t}\lambda_{x}/2}  \Bigg) \\
&= & g_{a,t}(N_{a,t}). 
\end{eqnarray*}

This concludes the proof. 
 \end{proof}
 
\subsection{Closed form of \texorpdfstring{\( g_{a,t}^{-1}(s) \)}{monotonic function} }
Now we calculate a closed form expression of \( N_{a,t} \). 
Setting the upper bound on confidence in the Theorem 4.1 in main paper to \( s\), we calculate the inverse in terms of \( N_{a,t} \),
\begin{eqnarray*}
8 (C_1 \beta + C_2)^2 \Bigg(   \frac{ 1}{\lambda+N_{a,t}\lambda_{x}/2}  \Bigg) &=&   s^2.
\end{eqnarray*}
Rearranging all the terms, we get
\begin{eqnarray*}
8 (C_1 \beta + C_2)^2   &=&  s^2 (\lambda+N_{a,t}\lambda_{x}/2) \\
(\lambda+N_{a,t}\lambda_{x}/2) &=&  \frac{8 (C_1 \beta + C_2)^2 }{s^2 }  \\
N_{a,t} &=&  \frac{16 (C_1 \beta + C_2)^2 }{s^2\lambda_{x} } - \frac{2\lambda}{\lambda_{x}}.   
\end{eqnarray*}

 Define
 \begin{equation}\label{eq:g_at}
 g_{a,t}^{-1}(s) =\frac{16 (C_1 \beta + C_2)^2 }{s^2\lambda_{x} } - \frac{2\lambda}{\lambda_{x}}.
 \end{equation}
 
\section{Simple Regret Analysis}
%Change the statement according to the main paper or according to notations and style used in the main paper
\begin{lemma}[Value of \( \beta \)] \label{lem:beta_effective_rank_r}
Assume the conditions in Theorem 4.1 and Lemma 4.2 in main paper. If \( \sum_{a \in [A]} g_{a,t}^{-1}(H_{a\epsilon})  = T - N_{\lambda}( A -1)\), then 

\begin{equation}
\label{eqn:beta}
 \beta    = \sqrt{\frac{\lambda_{x}(T - N_{\lambda}( A -1)) + 2 A\lambda}{ 16  C_1^2 H_{\epsilon} }} - \frac{C_2}{C_1}.   
\end{equation}

\end{lemma}

\begin{proof}

%Set \( \sum_{a \in [A])} g_{a,t}^{-1}(H_{a\epsilon} )  = T - A - N_{\lambda}( A -1) \) and solving for \( \beta \) gives,
We have 
\begin{eqnarray*}
 \nonumber
\sum_{a \in [A]} g_{a,t}^{-1}(H_{a\epsilon}  ) &=& T - N_{\lambda}( A -1). 
\end{eqnarray*}
By using Eqn. \eqref{eq:g_at},
\begin{eqnarray*}
\sum_{a \in [A]} \frac{16 (C_1 \beta + C_2)^2 }{H_{a\epsilon}^2\lambda_{x} } - \frac{2\lambda}{\lambda_{x}}&=&T - N_{\lambda}( A -1) \\
\frac{16 (C_1 \beta + C_2)^2 }{\lambda_{x} } \sum_{a \in [A]} \frac{ 1 }{H_{a\epsilon} ^2 } - \frac{2A\lambda}{\lambda_{x}}&=& T - N_{\lambda}( A -1).
\end{eqnarray*}
By using definition of \( H_{\epsilon} \),
\begin{eqnarray*}
\frac{16 (C_1 \beta + C_2)^2 H_{\epsilon} }{\lambda_{x} } - \frac{2A \lambda}{\lambda_{x}}&=& T - N_{\lambda}( A -1)
\end{eqnarray*}
Rearranging the terms,
\begin{eqnarray*}
16 (C_1 \beta + C_2)^2 H_{\epsilon}   &=& \lambda_{x}(T - N_{\lambda}( A -1)) + 2A\lambda \\
(C_1 \beta + C_2)^2  &=& \frac{\lambda_{x}(T - N_{\lambda}( A -1)) + 2A\lambda}{ 16  H_{\epsilon} } \\
 \beta    &=& \sqrt{\frac{\lambda_{x}(T - N_{\lambda}( A -1)) + 2A\lambda}{ 16 C_1^2 H_{\epsilon} }} - \frac{C_2}{C_1}. 
\end{eqnarray*}

\end{proof}

\subsection{Proof of Theorem 5.3 in main paper}
\label{subsec:MainTheoremProof}
Let \( [A] = \{1,...,A\} \). We define a feasible set \( A^{\prime}(x) \subseteq [A] \) such that elements of \( A^{\prime}(x) \) contain possible set of arms that may be pulled if context \( x \) was observed at all times \(AN_{\lambda} < t \leq T\). The set \(A^\prime(x)\) is used to discount the arms that will never be pulled with context \(x\).

\begin{proof}
    The proof broadly follows the same structure presented in Theorem 2 of \citet{hoffman2014correlation}. We will provide the simple regret bound at the recommendation of time \(T+1\), since the algorithm operates in a pure exploitation setting, the recommended arm \(\Omega(x_{T+2})\) will follow the same properties.

Fix \(x \in \sX\) such that \(x\) can be generated from the filtration. We define the event \(\sE_{a,t}(x)\) to be the event in which for arm \(a \leq A\), \(f_{a}(x)\) lies between the upper and lower confidence bounds given \(x_1, x_2, ..., x_{t-1}\) %where \(x_i, i \leq t-1\) is \(F_{i-1}\) measurable. 
More precisely,
\begin{equation*}
\sE_{a,t}(x) = \{L_{a,t}(x_t) \leq f_{a}(x) \leq U_{a,t}(x) | x_1, x_2, \cdots, x_{t-1} \}.
\end{equation*}

For events \(\sE_{a,t}\), from Theorem 4.1 of the main paper,
\begin{equation*}
    \PP (\sE_{a,t}(x)) \geq 1 - e^{-\beta^2}.
\end{equation*}

Let \( N_{a,T}\) denote the number of times each arm has been tried upto time \( T \). Clearly \( \sum_{a = 1}^A N_{a,T} = T\). Also, note that we try each arm at least \( N_\lambda \) number of times before we run our algorithm. We define event \(\sE\) as \(\sE := \bigcup_{a \leq A, AN_\lambda < t \leq T} \sE_{a,t}(x) \). By the union bound we can show that 
\begin{equation*}
    \PP (\sE) \geq 1-A (T - AN_\lambda) e^{-\beta^2}.
\end{equation*}

The next part of the proof works by contradiction. 

Let \(\epsilon > 0\). The recommended arm at the end of time \(T\) for context \(x\) is defined as follows: let \(t^\ast:= \argmin_{AN_\lambda < t \leq T} B_{J_t(x),t}(x)\) then the recommended arm is \(\Omega(x) := J_{t^\ast}(x)\).

Conditioned on event \(\sE\), we will assume that the event \(R_{\Omega(x)}(x) > \epsilon\) is true and arrive at a contradiction with high probability. Note that if \(R_{\Omega(x)}(x) > \epsilon\), the recommended arm \(\Omega(x)\) is necessarily sub-optimal (regret is zero for the optimal arm).

Define \( M_{a,T}(x) \) as number of times arm \(a \in [A] \) would be selected in \( AN_{\lambda} < t \leq T \), if we had seen context \( x\) at all those times. Hence, \( \sum_{a \in A^{\prime}(x)} M_{a,T}(x) = T - AN_{\lambda} \) . Also, note that \( N_{a,T}(x) = M_{a,T}(x) + N_{\lambda}\) for \( a \in A^{\prime}(x) \) and \( N_{a,T}(x) = N_{\lambda}\) otherwise.  Let \(t_a=t_a(x)\) be the last time instant for which arm \( a \in  A^{\prime}(x)\) may have been selected using the Contextual-Gap algorithm if context \(x \) was observed throughout. %\blue{It was observed throughout may covey a wrong thing. Could you modify the last statement?} 

The following holds for the recommended arm \(\Omega(x)\) with context \(x\):
\begin{equation*}
\begin{aligned}
    \min(0, s_{a,t_a}(x) - \Delta_a(x)) + s_{a,t_a}(x) &\geq B_{J_{t_a}(x),t_a}(x) \\
    	& \geq B_{\Omega(x), T+1}(x) \\
    	& \geq R_{\Omega(x)}(x) \\
        & > \epsilon.
\end{aligned}
\end{equation*}
Where the first inequality holds due to Lemma \ref{lem:mingapineq}, the second inequality holds by definition of \(B_{\Omega(x), T+1}\), the third inequality holds due to Lemma \ref{lem:BgeqR} and the last inequality holds due to the event \(R_{\Omega(x)} > \epsilon\).
The preceding inequality can also be written as 
\begin{equation*}
\begin{aligned}
s_{a,t_a}(x) > 2 s_{a,t_a}(x) - \Delta_a(x) > \epsilon, &\qquad \text{ if } \Delta_a(x) > s_{a,t_a}(x). \\
2 s_{a,t_a}(x) - \Delta_a(x) > s_{a,t_a}(x) >  \epsilon, &\qquad \text{ if } \Delta_a(x) < s_{a,t_a}(x).
\end{aligned}
\end{equation*}
This leads to the following bound on the confidence diameter of \(a \in [A] \),
\begin{equation*}
s_{a,t_a}(x) > \max(\frac{1}{2}(\Delta_a(x) + \epsilon), \epsilon) =: H_{a\epsilon}(x).
\end{equation*}
For any arm \(a\), we consider the final number of arm pulls \(M_{a,T}(x) + N_{\lambda}\). From Lemma 5.2 of the main paper we can write, using the strict monotonicity and there by invertibility of \(g_{a,T}\), with probability at least \(1-\delta\) as

%\blue{There is a problem here. We need to get \( g_{a,T}^{-1}\) in terms of \(M_{a,t} \) and not in terms of \( N_{a,t}\), so that we can apply the following. I think only constants are going to change. }
\begin{equation*}
\begin{aligned}
M_{a,T}(x) + N_{\lambda} & \leq g_{a,T}^{-1}(s_{a,t_a}(x))  \\
	& < g_{a,T}^{-1}(H_{a\epsilon}(x)) \\
    & \leq g_{a,T}^{-1}(H_{a\epsilon}),
\end{aligned}
\end{equation*}

%\blue{Write explanation here}
where \( H_{a\epsilon} = \inf_{x} H_{a\epsilon}(x)\). 
Last two equations hold as \(g_{a,T}\) is a monotonically decreasing function. By summing both sides with respect to \(a \in A^{\prime}(x)\) we can write
\begin{eqnarray*}
T - AN_{\lambda} + |A^\prime(x)|N_{\lambda}  &<& \sum_{a \in A^{\prime}(x)} g_{a,T}^{-1}(H_{a\epsilon}),
\end{eqnarray*}
We can make RHS even bigger by adding terms \( a \in [A] \)\textbackslash \( A^{\prime}(x) \). Hence, we get
\begin{eqnarray*}
T - (A - |A^\prime(x)|)N_{\lambda}  &<& \sum_{a \in [A]} g_{a,T}^{-1}(H_{a\epsilon}).
%T - (A - |A^\prime(x)|)N_{\lambda}  &<& \sum_{a \in [A]} g_{a,T}^{-1}(H_{a\epsilon}).
\end{eqnarray*}
We can make LHS even smaller by noting that minimum value of \(|A^\prime(x)| \) is one.
\begin{eqnarray*}
T - AN_{\lambda}  + N_{\lambda}  &<& \sum_{a \in [A]} g_{a,T}^{-1}(H_{a\epsilon}).
\end{eqnarray*}
Rearranging the terms, we get
\begin{eqnarray*}
T - AN_{\lambda}  + N_{\lambda}  &<& \sum_{a \in [A]} g_{a,T}^{-1}(H_{a\epsilon}) \\
T - N_{\lambda}( A -1)  &<& \sum_{a \in [A]} g_{a,T}^{-1}(H_{a\epsilon}).
\end{eqnarray*}
which contradicts our definition of \(g_{a,T}\) in the theorem statement. Therefore \(R_{\Omega(x)}(x) \leq \epsilon\).

From the preceding argument we have that if \(\sum_{a \in [A]} g_{a,T}^{-1}(H_{a\epsilon}) \leq T - N_{\lambda}( A -1) \), then for any \(x \in \sX\) generated from the filtration,
    
\begin{equation*}
   \PP (R_{\Omega(x)} < \epsilon | x) \geq 1-A (T - AN_\lambda) e^{-\beta^2} - A\delta.
\end{equation*}

In the above equation, \(1-A (T - AN_\lambda) e^{-\beta^2}\) is from the event \(\sE\) and \(1 - A\delta\) is due to the fact that the monotonic upper bounds holds only with probability \(1-\delta\) for each of the arms.
Setting \(\beta\) such that \(\sum_{a \in [A]} g_{a,T}^{-1}(H_{a\epsilon}) = T- N_{\lambda}( A -1) \) (See Lemma \ref{lem:beta_effective_rank_r}), we have for

\begin{equation*}
 \beta    = \sqrt{\frac{\lambda_{x}(T - N_{\lambda}( A -1)) + 2A\lambda}{ 16 C_1^2 H_{\epsilon} }} - \frac{C_2}{C_1},
\end{equation*}
that
\begin{equation*}
    \PP (R_{\Omega(x)} < \epsilon | x) \geq 1-A (T - AN_\lambda) e^{-\beta^2} - A\delta,
\end{equation*}

for \(C_1 = \rho \sqrt{2}\) and \(C_2 = \rho \sqrt{\sum_{\tau=2}^T \ln(1+ \frac{1}{\lambda} \hat{\sigma}_{a,\tau-1}(x_\tau) )} + \sqrt{\lambda} \n{f_a}_\sH\).

Since \(C_2\) depends on \(T\), to complete the proof and validity of the bound, we will show that \(C_2\) grows logarithmically in \(T\). When assumption \ref{Alambda} holds and \(\n{\phi(x)} \leq 1\), similar to the analysis in \citet{abbasi2011improved,durand2017streaming}, we have
\begin{equation*}
\begin{aligned}
	C_2 &= \rho \sqrt{\sum_{\tau=2}^T \ln(1+ \frac{1}{\lambda} \hat{\sigma}_{a,\tau-1}(x_\tau) )} + \sqrt{\lambda} \n{f_a}_\sH\\
    &= \rho \sqrt{\sum_{\tau=2}^T \ln(1+ \frac{1}{\lambda} \phi(x_\tau)^T(I + \frac{1}{\lambda} K_{a,\tau-1} )^{-1}\phi(x_\tau) ) } \\
    &  + \sqrt{\lambda} \n{f_a}_\sH \\
    &= \rho \sqrt{\ln(\det(I + \frac{1}{\lambda} K_{a,T} ) ) } + \sqrt{\lambda} \n{f_a}_\sH \\
    & \leq \rho \sqrt{d^\ast \ln \bigg(\frac{1}{d^\ast}\bigg(1+\frac{T}{\lambda}\bigg) \bigg)} + \sqrt{\lambda} \n{f_a}_\sH.
\end{aligned}
\end{equation*}
Since \(C_2\) depends on \(\sqrt{\ln(T)}\), we fix  \(C_2 = O(\rho  \sqrt{\ln(T)}) \). As \(T \rightarrow \infty\) the RHS of the probability bound goes to unity and we have the resulting theorem.

\end{proof}

\subsection{Lemmas over event \texorpdfstring{\(\sE\)}{E}}
For arm \(a\) at time \(t\), we define event \(\sE_{a,t}\) as
\begin{equation*}
\sE_{a,t}(x) = \{L_{a,t}(x_t) \leq f_{a}(x) \leq U_{a,t}(x) | x_1, x_2, \cdots, x_{t-1} \}.
\end{equation*}
We define event \(\sE\) as \(\sE := \bigcup_{a \leq A, AN_\lambda < t \leq T} \sE_{a,t}(x) \)

The following theorems operate under the assumption the event \(\sE\) holds. We provide two properties of the terms in the algorithm that will be of help in the proofs:
\begin{itemize}
    \item \(B_{J_t(x)} = U_{j_t(x),t}(x) - L_{J_t(x),t}(x)\)
    \item \(U_{a,t}(x) = L_{a,t}(x) + s_{a,t}(x)\)
\end{itemize}

\begin{lemma}
\label{lem:BgeqR}
Over event \(\sE\), for any sub-optimal arm \(a(x) \neq a^\ast(x)\) at any time \(t \leq T\), the simple regret of pulling that arm is upper bounded by the \(B_{a,t}(x)\), %i.e., \(B_{a,t}(x) \geq r_a(x)\)
\end{lemma}

\begin{proof}
\begin{equation*}
\begin{aligned}
B_{a,t}(x) &= \max_{i \neq a} U_{i,t}(x) - L_{a,t}(x) \\
& \geq \max_{i \neq a} f_{i}(x) - f_{a,t}(x) = f^{\ast}(x) - f_{a}(x) = R_{a}(x).
\end{aligned}
\end{equation*}
The first inequality holds due to the definition of event \(\sE\) and the equality holds since we are only considering sub-optimal arms.
\end{proof}

Note that the preceding lemma need not hold for the optimal arm, for which \(R_{a} (x) = 0\) and it is not necessary that \(B_{a,t}(x) \geq 0\).

\begin{lemma}
    \label{lem:ljlJujuJinequality}
Consider the contextual bandit setting proposed in the main paper. Over event \(\sE\), for any time \(t\) and context \(x \in \sX\), the following statements hold for the arm \(a = a_t\) to be selected:
\begin{equation*}
\begin{aligned}
\text{if } a = j_t(x), &\text{ then } L_{j_t(x),t}(x) \leq L_{J_t(x),t}(x), \\
\text{if } a = J_t(x), &\text{ then } U_{j_t(x),t}(x) \leq U_{J_t(x),t}(x).
\end{aligned}
\end{equation*}
\end{lemma}
\begin{proof}
We consider two cases based on which of the two candidate arms \(j_t(x), J_t(x)\) is selected.

\textbf{Case 1:} \(a = j_t(x)\) is selected. The proof works by contradiction.
Assume that \(L_{j_t(x),t}(x) > L_{J_t(x),t}(x)\). From the arm selection rule we have \(s_{j_t(x),t}(x) \geq s_{J_t(x),t}(x)\). Based on this we can deduce that \(U_{j_t(x),t}(x) \geq U_{J_t(x),t}(x)\). As a result, 
\begin{equation*}
\begin{aligned}
B_{j_t(x),t}(x) &  = \max_{i \neq j_t(x)} U_{i,t}(x) - L_{j_t(x),t}(x) \\ 
& < \max_{i \neq J_t(x)} U_{i,t}(x) - L_{J_t(x),t}(x) = B_{J_{t}(x),t}(x).
\end{aligned}
\end{equation*}
The above inequality holds because the arm \(j_t(x)\) must necessarily have the highest upper bound over all the arms. However, this contradicts the definition of \(B_{J_t(x),t}(x)\) and as a result it must hold that \(L_{j_t(x),t}(x) \leq L_{J_t(x),t}(x)\).

\textbf{Case 2:} \(a = J_t(x)\) is selected. The proof works by contradiction.
Assume that \(U_{j_t(x),t}(x) > U_{J_t(x),t}(x)\). From the arm selection rule we have \(s_{J_t(x),t}(x) \geq s_{j_t(x),t}(x)\). Based on this we can deduce that \(L_{J_t(x),t}(x) \leq L_{j_t(x),t}(x)\). As a result, similar to Case 1, 
\begin{equation*}
\begin{aligned}
B_{j_t(x),t}(x) & = \max_{j \neq j_t(x)} U_{j,t}(x) - L_{j_t(x),t}(x) \\ 
& < \max_{j \neq J_t(x)} U_{j,t}(x) - L_{J_t(x),t}(x) = B_{J_{t}(x),t}(x).
\end{aligned}
\end{equation*}
The above inequality holds because the arm \(j_t(x)\) must necessarily be have the highest upper bound over all the arms. However, this contradicts the definition of \(B_{J_t(x),t}(x)\) and as a result it must hold that \(U_{j_t(x),t}(x) \leq U_{J_t(x),t}(x)\).
\end{proof}

\begin{corollary}
    \label{cor:sBinequality}
For context \(x\), if arm \(a = a_t(x)\) is pulled at time \(t\), then \(B_{J_t(x),t}(x)\) is bounded above by the uncertainty of arm \(a\), i.e., 
\begin{equation*}
B_{J_t(x),t}(x) \leq s_{a,t}(x).
\end{equation*}
\end{corollary}

\begin{proof}
By construction of the algorithm \(a \in \{j_t(x),J_t(x)\}\). If \(a = j_t(x)\), then using the definition of \(B_{J_t(x),t}(x)\) and Lemma \ref{lem:ljlJujuJinequality}, we can write
\begin{equation*}
\begin{aligned}
B_{J_t(x),t}(x) &= U_{j_t(x),t}(x) - L_{J_t(x),t}(x) \\
	& \leq U_{j_t(x),t}(x) - L_{j_t(x),t}(x) = s_{a,t}(x).
\end{aligned}
\end{equation*}
Similarly, for \(a = J_t(x)\), 
\begin{equation*}
\begin{aligned}
B_{J_t(x),t}(x) &= U_{j_t(x),t}(x) - L_{J_t(x),t}(x) \\
	& \leq U_{J(x),t}(x) - L_{J_t(x),t}(x) = s_{a,t}(x).
\end{aligned}
\end{equation*}
\end{proof}

\begin{lemma}
\label{lem:mingapineq}
On event \(\sE\), for any time \(t \leq T\) and for arm \(a = a_t(x)\) the following bounds hold for the minimal gap
\begin{equation*}
    B_{J_t(x),t}(x) \leq \min(0, s_{a,t}(x) - \Delta_a(x)) + s_{a,t}(x).
\end{equation*}
\end{lemma}

\begin{proof}
    The arm to be pulled is restricted to \(a \in \{j_t(x),J_t(x)\}\). The optimal arm for the context \(x\) at time \(t\) can either belong to \(\{j_t(x), J_t(x)\}\) or be equal to some other arm. This results in 6 cases:
    \begin{enumerate}
        \item \(a = j_t(x), a^\ast = j_t(x)\)
        \item \(a = j_t(x), a^\ast = J_t(x)\)
        \item \(a = j_t(x), a^\ast \notin \{j_t(x), J_t(x)\}\)
        \item \(a = J_t(x), a^\ast = j_t(x)\)
        \item \(a = J_t(x), a^\ast = J_t(x)\)
        \item \(a = J_t(x), a^\ast \notin \{j_t(x), J_t(x)\} \)
    \end{enumerate}
	We define \(f^\ast(x):= f_{a^\ast}(x)\) as the expected reward associated with the best arm and \(f_{(a)}(x)\) as the expected reward of the \(a^{\text{th}}\) best arm.

    \textbf{Case 1:} The following sequence of inequalities holds:
    \begin{eqnarray*}
        f_{(2)}(x) &\geq & f_{J_t(x)}(x) \\
        &\geq& L_{J_t(x),t}(x) \\
        &\geq& L_{j_t(x),t}(x) \\
        &\geq& f_a(x) - s_{a,t}(x).
    \end{eqnarray*}
    The first inequality follows from the assumption that \(a = a^\ast = j_t(x)\), the chosen and optimal arm has the highest upper confidence bound, and therefore, the expected reward of arm \(J_t(x)\) can be at most that of the second best arm. The second inequality follows from event \(\sE\), the third inequality follows from \ref{lem:ljlJujuJinequality}. The last inequality follows from event \(\sE\). Using the above string of inequalities and the definition of \(\Delta_a(x)\), we can write 
    \begin{equation*}
        s_{a,t} - (f_{a}(x) - f_{(2)}(x)) = s_{a,t} - \Delta_a(x) \geq 0.
    \end{equation*}
    The result holds for case 1 with the application of Corollary \ref{cor:sBinequality}.

    \textbf{Case 2:} \(a = j_t(x), a^\ast = J_t(x)\). We can write 
    \begin{equation*}
        \begin{aligned}
            B_{J_t(x),t}(x) &= U_{j_t(x),t}(x) - L_{J_t(x),t}(x) \\
                            & \leq f_{j_t(x)}(x) + s_{j_t(x),t}(x) \\
                            &   - f_{J_t(x)}(x) + s_{J_t(x),t}(x) \\
                            & \leq f_a(x) - f^\ast(x) + 2s_{a,t}(x).
        \end{aligned}
    \end{equation*}
    The first inequality follows from event \(\sE\) and the second inequality holds because the selected arm has a larger uncertainty. From the definition of \(\Delta_a(x)\),
    \begin{equation*}
        \begin{aligned}
            B_{J_t(x),t}(x) &\leq 2s_{a,t}(x) - \Delta_a(x) \\
                            & \leq s_{a,t}(x) + \min(0, s_{a,t} - \Delta_a(x)).
        \end{aligned}
    \end{equation*}
    Where the inequality follows from Corollary \ref{cor:sBinequality}.

    \textbf{Case 3:} \(a = j_t(x), a^\ast \notin \{j_t(x), J_t(x)\}\). We can write the following sequence of inequalities 
        \begin{equation*}
            \begin{aligned}
                f_{j_t(x)}(x) + s_{j_t(x),t}(x) \geq U_{j_t(x),t}(x) \geq U_{a^\ast} \geq f^\ast.
            \end{aligned}
        \end{equation*}
        The first and third inequalities hold due to event \(\sE\), the second inequality holds by definition as \(j_t(x)\) has the highest upper bound on any arm other than \(J_t(x)\) neither of which is the optimal arm in this case. From the first and last inequalities, we obtain 
        \begin{equation*}
            s_{a,t}(x) - (f^\ast - f_{a,t}(x)) \geq 0,
        \end{equation*}
        or \(s_{a,t}(x) - \Delta_a(x) \geq 0\). The result follows from Corollary \ref{cor:sBinequality}.
        
        \textbf{Case 4:} \(a = J_t(x), a^\ast = j_t(x)\). We can write 
    \begin{equation*}
        \begin{aligned}
            B_{J_t(x),t}(x) &= U_{j_t(x),t}(x) - L_{J_t(x),t}(x) \\
                            & \leq f_{j_t(x)}(x) + s_{j_t(x),t}(x) \\
                            & - f_{J_t(x)}(x) + s_{J_t(x),t}(x) \\
                            & \leq f_a(x) - f^\ast(x) + 2s_{a,t}(x).
        \end{aligned}
    \end{equation*}
    The first inequality follows from event \(\sE\) and the second inequality holds because the selected arm has a larger uncertainty. From the definition of \(\Delta_a(x)\),
    \begin{equation*}
        \begin{aligned}
            B_{J_t(x),t}(x) &\leq 2s_{a,t}(x) - \Delta_a(x) \\
                            & \leq s_{a,t}(x) + \min(0, s_{a,t} - \Delta_a(x)).
        \end{aligned}
    \end{equation*}
    Where the inequality follows from Corollary \ref{cor:sBinequality}.

        \textbf{Case 5:} \(a = J_t(x), a^\ast = J_t(x)\). The following sequence of inequalities holds:
        \begin{eqnarray*}
        f_{a}(x) + s_{a,t}(x) &\geq&  U_{J_t(x),t}(x)\\ 
        &\geq& U_{j_t(x),t}(x)\\
        &\geq& f_{j_t(x)}(x)  \\
        &\geq& f_{(2)}(x).
        \end{eqnarray*}
        The first and third inequalities follow from event \(\sE\), the second inequality is a consequence of Lemma \ref{lem:ljlJujuJinequality}, the fourth inequality follows from the fact that since \(J_t(x)\) is the optimal arm, the upper bound and the arm selected should be as good as the second arm.
        Using the above chain of inequalities, we can write
        \begin{equation*}
        s_{a,t}(x) - (f_{(2)}(x) - f_a(x)) = s_{a,t}(x) - \Delta_a(x) \geq 0.
        \end{equation*}
        
        \textbf{Case 6:} \(a = J_t(x), a^\ast \notin \{j_t(x), J_t(x)\}\). We can write the following sequence of inequalities 
        \begin{equation*}
            \begin{aligned}
                f_{J_t(x)}(x) + s_{J_t(x),t}(x) \geq U_{J_t(x),t}(x) \geq U_{a^\ast, t}(x) \geq f^\ast.
            \end{aligned}
        \end{equation*}
        The first and third inequalities hold due to event \(\sE\), the second inequality holds by definition as \(J_t(x)\) has the highest upper bound on any arm when \(a = J_t(x)\) due to Lemma \ref{lem:ljlJujuJinequality} and \(J_t(x)\) is not optimal in this case. From the first and last inequalities, we obtain 
        \begin{equation*}
            s_{a,t}(x) - (f^\ast - f_{a,t}(x)) \geq 0,
        \end{equation*}
        or \(s_{a,t}(x) - \Delta_a(x) \geq 0\). The result follows from Corollary \ref{cor:sBinequality}.

\end{proof}

\section{Experimental Details, Discussion and Additional Experimental Results}
\label{sec:DetailedExperiments}

We present average simple regret comparisons of the Contextual-Gap algorithm against five baselines:
\begin{enumerate}[topsep=0pt, noitemsep]
    \item Uniform Sampling: We equally divide the exploration budget \(T\) among arms and learn a reward estimating function \( f_a: \sX \rightarrow \RR \) for each of the arm during exploration. During the exploitation phase, we select the best arm based on estimated reward function \( f_a \). 
    \item Epsilon Greedy: At every step, we select the best arm (according to estimated \( f_a \) ) with probability \( 1 - \epsilon_t\) and other arms with probability \( \epsilon_t\). We use   \(\epsilon_t = 0.99^t\), where \( t \) is the time step.
    \item Kernel-UCB: We implement kernel-UCB from \citet{valko2013finite} for both exploration and exploitation.
    \item Kernel-UCB-Mod: We implement kernel-UCB from \citet{valko2013finite} for exploration but use best arm based on estimated reward function \( f_a \) for exploitation.  
    \item Kernel-TS: We use kernelized version of Thompson Sampling from \citet{chowdhury2017kernelized}. 
    %\item Contextual-Gap: This is a proposed algorithm as shown in algorithm \ref{alg:contextualGap}. 
\end{enumerate}

\begin{figure}
\centering
        \begin{subfigure}[h]{0.495\textwidth}
            \centering
        \includegraphics[width=0.9\textwidth]{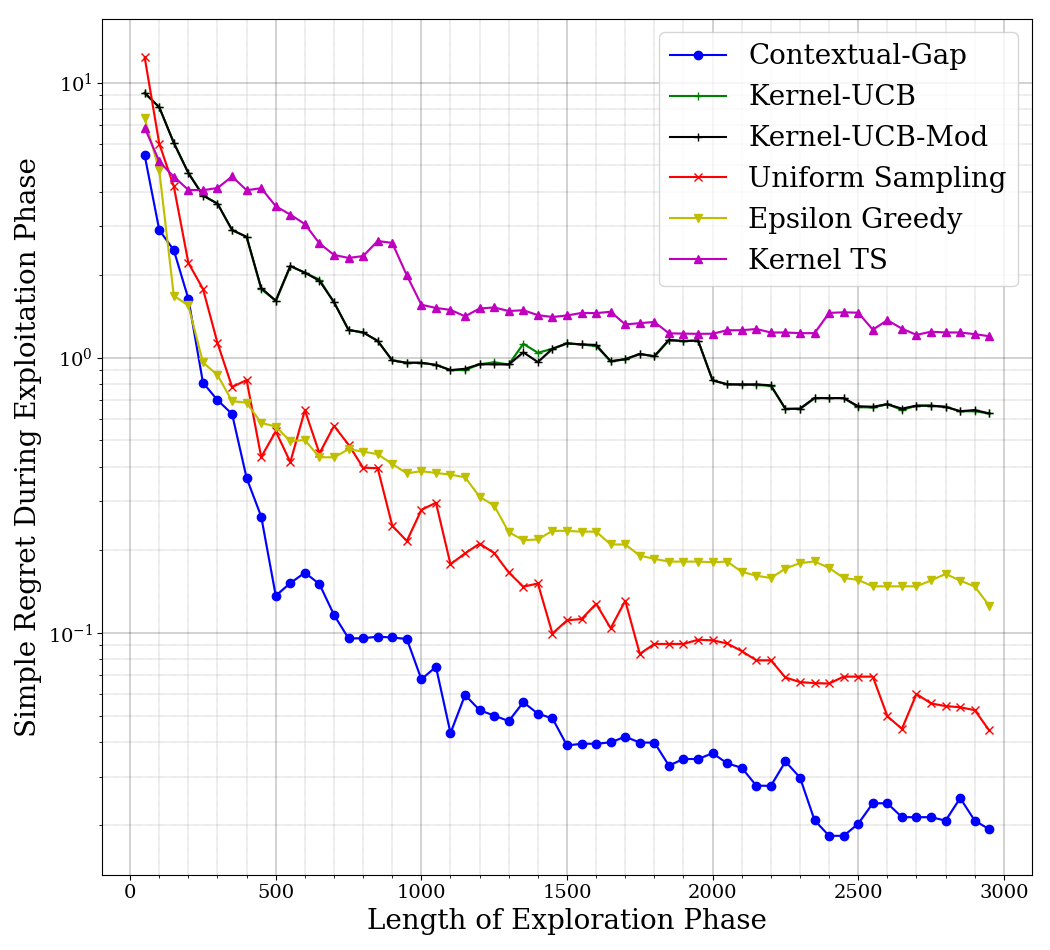}
        \caption{Spacecraft dataset}
        \label{fig:Spacecraft_alpha_01}
    \end{subfigure}
\caption{Simple Regret evaluation with \(\alpha=0.1\)}
\label{fig:alpha_01}
\end{figure}  
\begin{figure}
\centering
        \begin{subfigure}[h]{0.495\textwidth}
            \centering
        \includegraphics[width=0.9\textwidth]{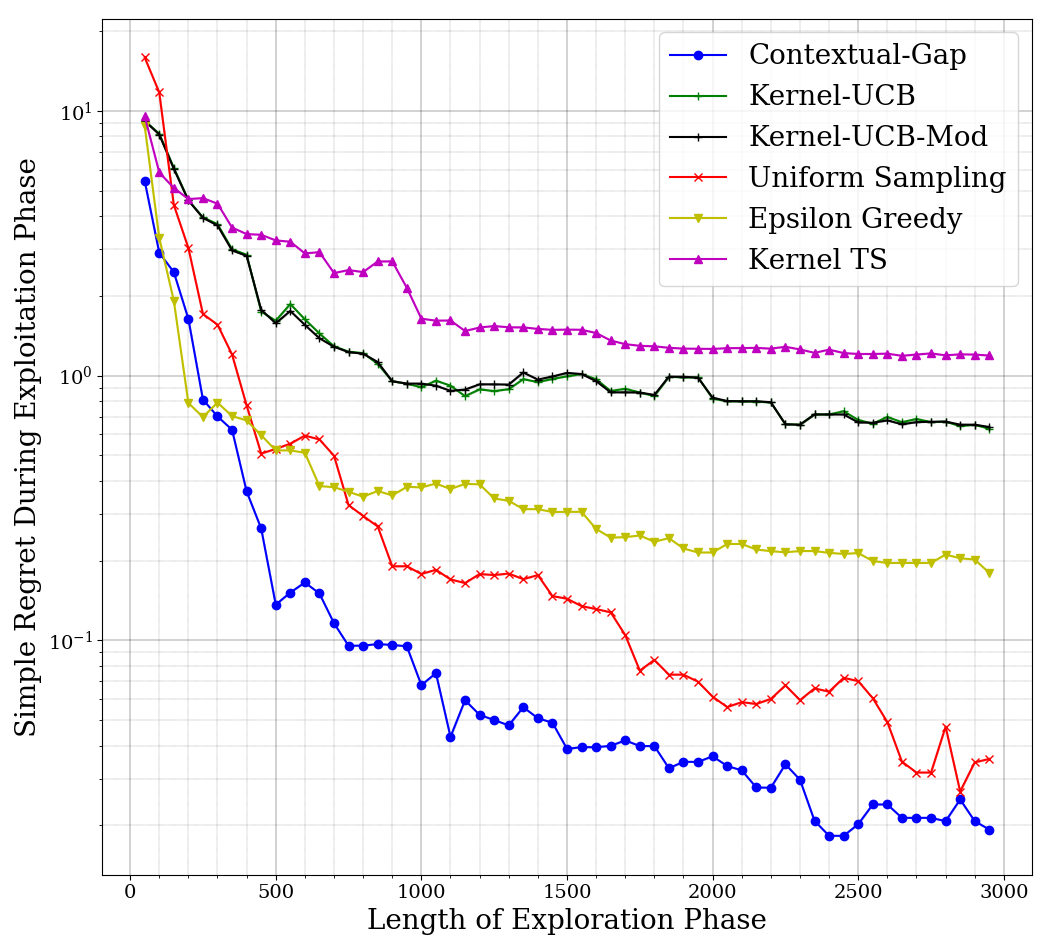}
        \caption{Spacecraft dataset}
        \label{fig:Spacecraft_alpha_05}
    \end{subfigure}
\caption{Simple Regret evaluation with \(\alpha=0.5\)}
\label{fig:alpha_05}
\end{figure}  

\begin{figure}
\centering
        \begin{subfigure}[h]{0.495\textwidth}
            \centering
        \includegraphics[width=0.9\textwidth]{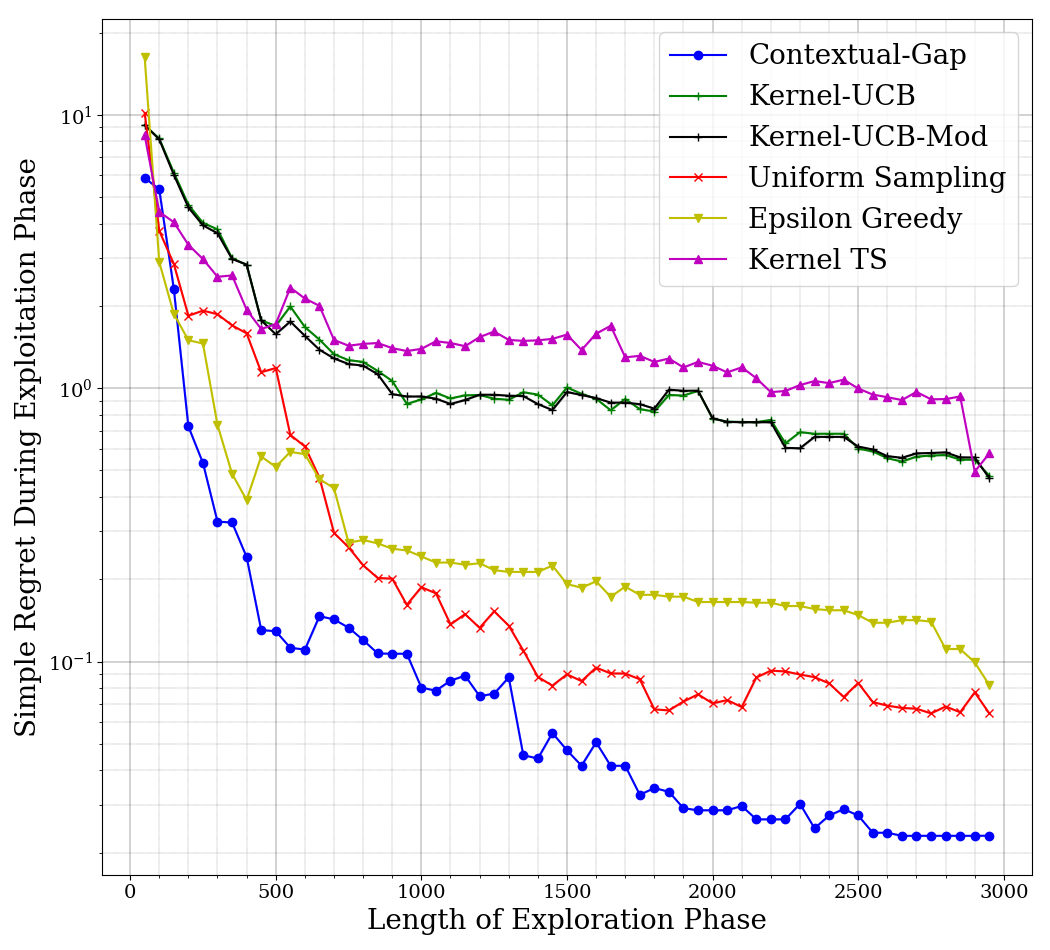}
        \caption{Spacecraft dataset}
        \label{fig:Spacecraft_alpha_2}
    \end{subfigure}
\caption{Simple Regret evaluation with \(\alpha=2\)}
\label{fig:alpha_2}
\end{figure}

The algorithm was implemented with the best arm chosen with a history of one i.e., \(\Omega(x_t) = J_{T}(x_t) \). For speed and scalability in implementation, the kernel inverse for arm \(a\), \((K_{a,t}+\lambda I_{a,t})^{-1}\) and the kernel vector \(k_{a,t}(x)\) updates were implemented as rank one updates. To tune kernel bandwidth and regularization parameters, we use following procedure: The dataset was split into two parts for hold-out (HO) and evaluation (EV). Each part was further split into two phases: exploration and exploitation. The value of \(T\) selected in both the hold-out and evaluation datasets were of similar magnitude. On the hold-out dataset, a grid search was used to set the tuning parameters by optimizing the average simple regret of the exploitation phase. The tuned parameters were used with the evaluation datasets to generate the plots. The code is available online to reproduce all results  \footnote{The code to reproduce our results is available at \url{https://bit.ly/2Msgdfp}}. As our implementation performs rank one updates of the kernel matrix and its inverse, our algorithm has \( O(T^2) \) as both computational and memory complexity in the worst case scenario, where \( T \) is the length of the exploration phase.

The following plots (Figures \ref{fig:alpha_01}, \ref{fig:alpha_05} and \ref{fig:alpha_2}) provide results for different \(\alpha\) values. Note that, the hyper parameter computed with cross validation for \(\alpha=1\) were retained for the evaluation runs of different values of \(\alpha\). It can be seen that Contextual Gap performs consistently better for all the datasets under consideration. 

The spacecraft magnetometer dataset is non-i.i.d in nature due to which it is not possible to generate confidence intervals by shuffling the data. Due to this, only the individual, non-averaged simple regret curves are presented.

A comparison of the average simple regret variation of Contextual gap with a history of the past 25 data points (instead of 1) is shown in Figure \ref{fig:history_comp}. It can be seen that there exists only minor differences in contextual gap runs with history.

\begin{figure}
    \centering
    \includegraphics[width=0.9\linewidth]{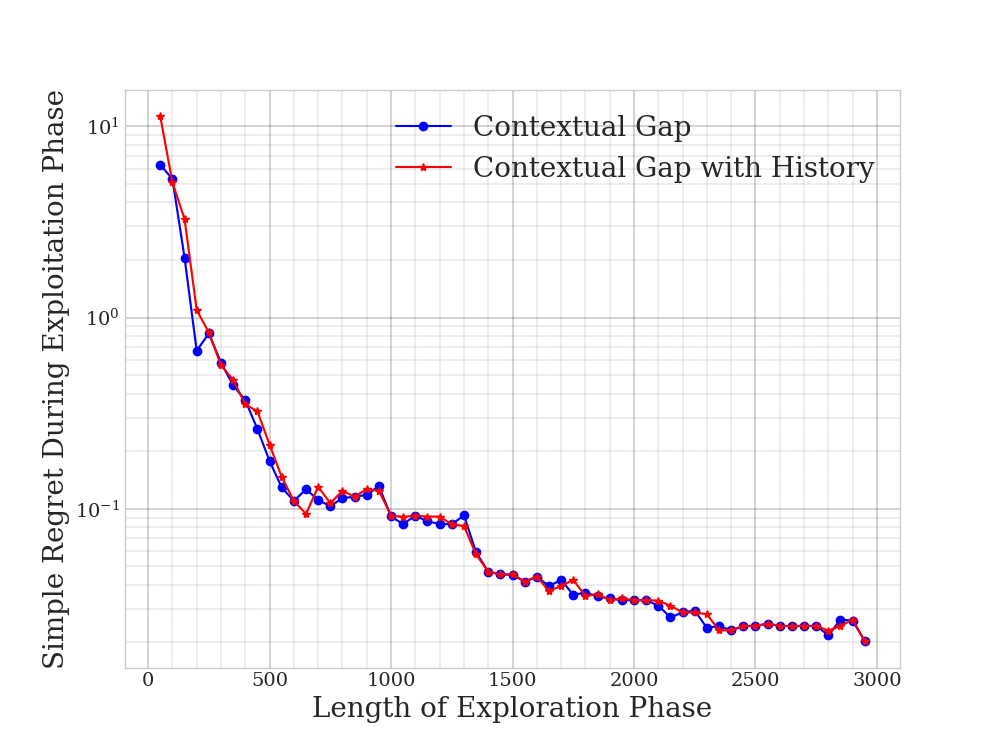}
    \caption{Comparison with Contextual Gap algorithm with recent history for Spacecraft dataset}
    \label{fig:history_comp}
\end{figure}

% In Figure \ref{fig:worst_case_SR}, we show worst case simple regret among all the data present in exploitation phase. 
The contextual gap algorithm presented is a solution to simple regret minimization, and not average simple regret minimization. Hence,we present the worst case simple regret among all the data present in the exploitation phase as additional empirical evidence of simple regret minimization ( Figure \ref{fig:worst_case_SR}).
 \begin{figure}
    \centering
    \includegraphics[width=0.9\linewidth]{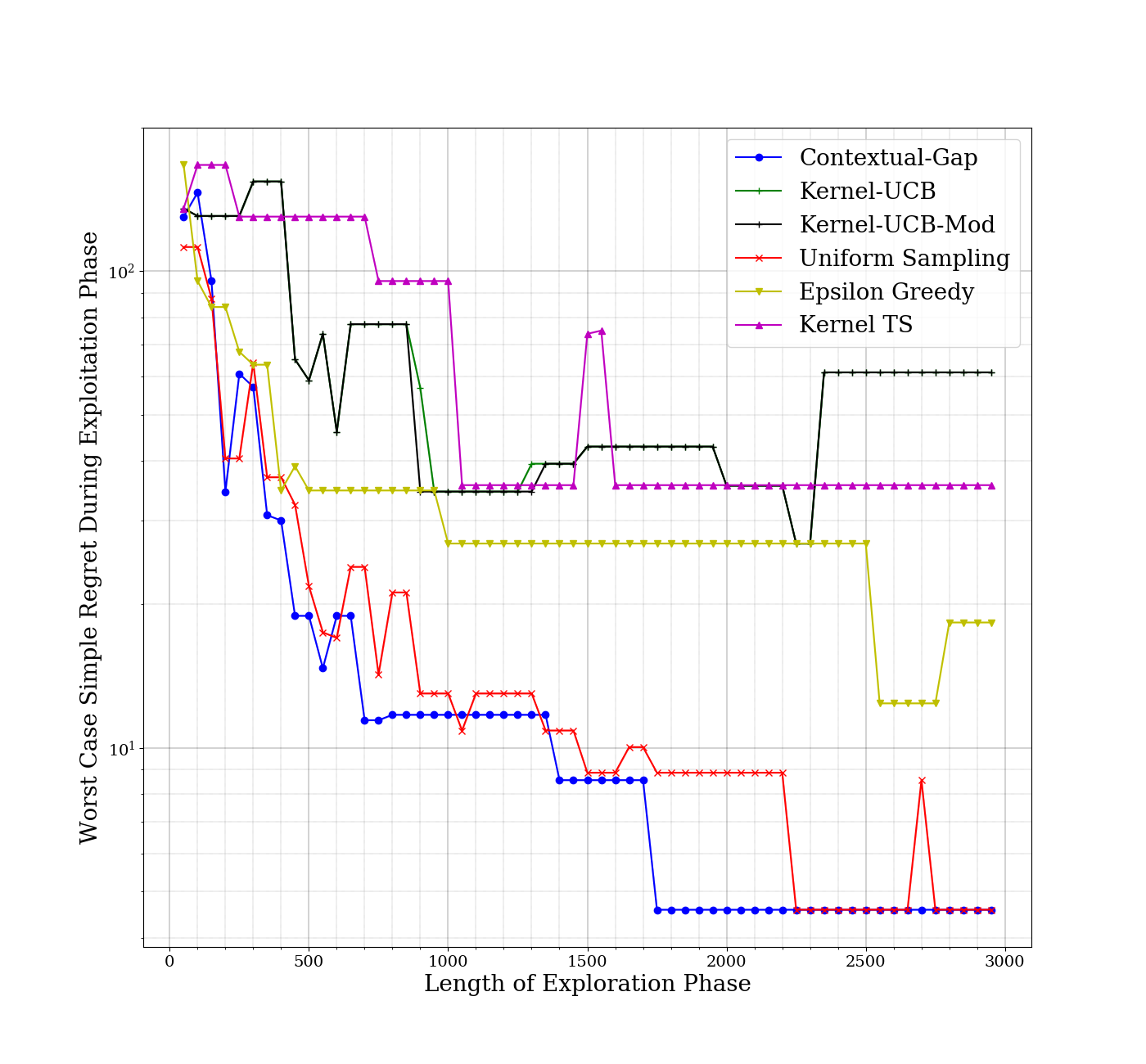}
    \caption{Worst Case Simple Regret Evaluation on Spacecraft Magnetic Field Dataset}
    \label{fig:worst_case_SR}
\end{figure}